\documentclass[12pt, a4paper]{article}

\usepackage{amsmath,amsfonts,color,amssymb}
\usepackage{amsthm}
\usepackage[T1]{fontenc} 
\usepackage{fullpage}
\usepackage{graphicx,epsf}
\usepackage{subfig} 
\usepackage{mathabx}

\newtheorem{theorem}{Theorem}
\newtheorem{proposition}{Proposition}
\newtheorem{corollary}{Corollary}

\newtheorem{lemma}{Lemma}
\newtheorem{remark}{Remark}

\newcommand{\trace}{\textrm{Tr}}

\begin{document}
	
	\title{A Semi-Definite Programming approach to low dimensional embedding for unsupervised clustering}
	\author{St\'ephane Chr\'etien \footnote{National Physical Laboratory,
			Mathematics and Modelling, Hampton Road, 
			Teddington, TW11 OLW, UK} \ , Cl\'ement Dombry \footnote{Laboratoire de Math\'ematiques de Besan\c{c}on, UMR CNRS 6623, Universit\'e de Bourgogne Franche-Comt\'e, 16 route de Gray, 25030 Besan{\c c}on cedex, France. Email: clement.dombry@univ-fcomte.fr} \ and Adrien Faivre \footnote{Digital Surf, 16 Rue Lavoisier, 25000 Besan{\c c}on, France}} 
	\maketitle
	\abstract{This paper proposes a variant of the method of Gu\'edon and Verhynin for estimating the cluster matrix in the Mixture of Gaussians framework via Semi-Definite Programming. A clustering oriented embedding is deduced from this estimate. The procedure is suitable for very high dimensional data because it is based on pairwise distances only. Theoretical garantees are provided and an eigenvalue optimisation approach is proposed for computing the embedding. The performance of the method is illustrated 
		via Monte Carlo experiements and comparisons with other embeddings from the literature. } 
	
	\section{Introduction}
	\subsection{Motivations}
	Low dimensional embedding is a key to many modern data analysis procedures. The main underlying idea is that the data is better understood after extracting the main features of the samples. Based on a compressed description from a few extracted features, the individual samples can then be projected, visualized or clustered more reliably and efficiently. 
	
	The main embedding techniques available nowadays are PCA \cite{jolliffe2002principal} and its robust version \cite{candes2011robust}, random embeddings \cite{johnson1984extensions}, (see also the recent \cite{cannings2015random} for supervised classification),
	Laplacian Eigenmap \cite{belkin2001laplacian}, Maximum Variance Unfolding/Semi-Definite embedding \cite{weinberger2006unsupervised}, \ldots The first two techniques in 
	this list are linear embeddings methods, whereas the other are nonlinear in nature. 
	
	\smallskip
	
	In modern data science, the samples may lie in very high dimensional spaces. Our main objective in the present paper is to propose a technique for a low dimensional representation which aims at preparing the data for unsupervised clustering at the same time. Combining the goals of projecting and clustering is not new.
	This is achieved in particular by spectral clustering \cite{von2007tutorial}
	\cite[Chapter 3]{bandeira2015ten}. The SemiDefinite embedding technique 
	in \cite{linial1995geometry} is also motivated by clustering purposes. Spectral clustering is based on a Laplacian matrix constructed from the pairwise distances of the samples and whose second eigenvector is 
	proved to separate the data into two clusters using the normalized cut 
	criterion. The second eigenvector is called the Fiedler vector. The analysis 
	is usually presented from the perspective of Cheeger's relaxation and a 
	clever randomized algorithm \cite[Chapter 3]{bandeira2015ten}.
	Clustering into more than two
	groups can also be performed using a higher order Cheeger theory \cite{lee2014multiway}, a direction which has not been much explored in practice yet.
	
	\smallskip
	
	A frequent way to illustrate non-linear low dimensional embedding such 
	as Diffusion Maps is shown in
	Figure \ref{NLEmb}. 
	\begin{center}
		\begin{figure}[htb]
			\subfloat[Original 3D Cluster]{%
				\includegraphics[width=.4\textwidth]{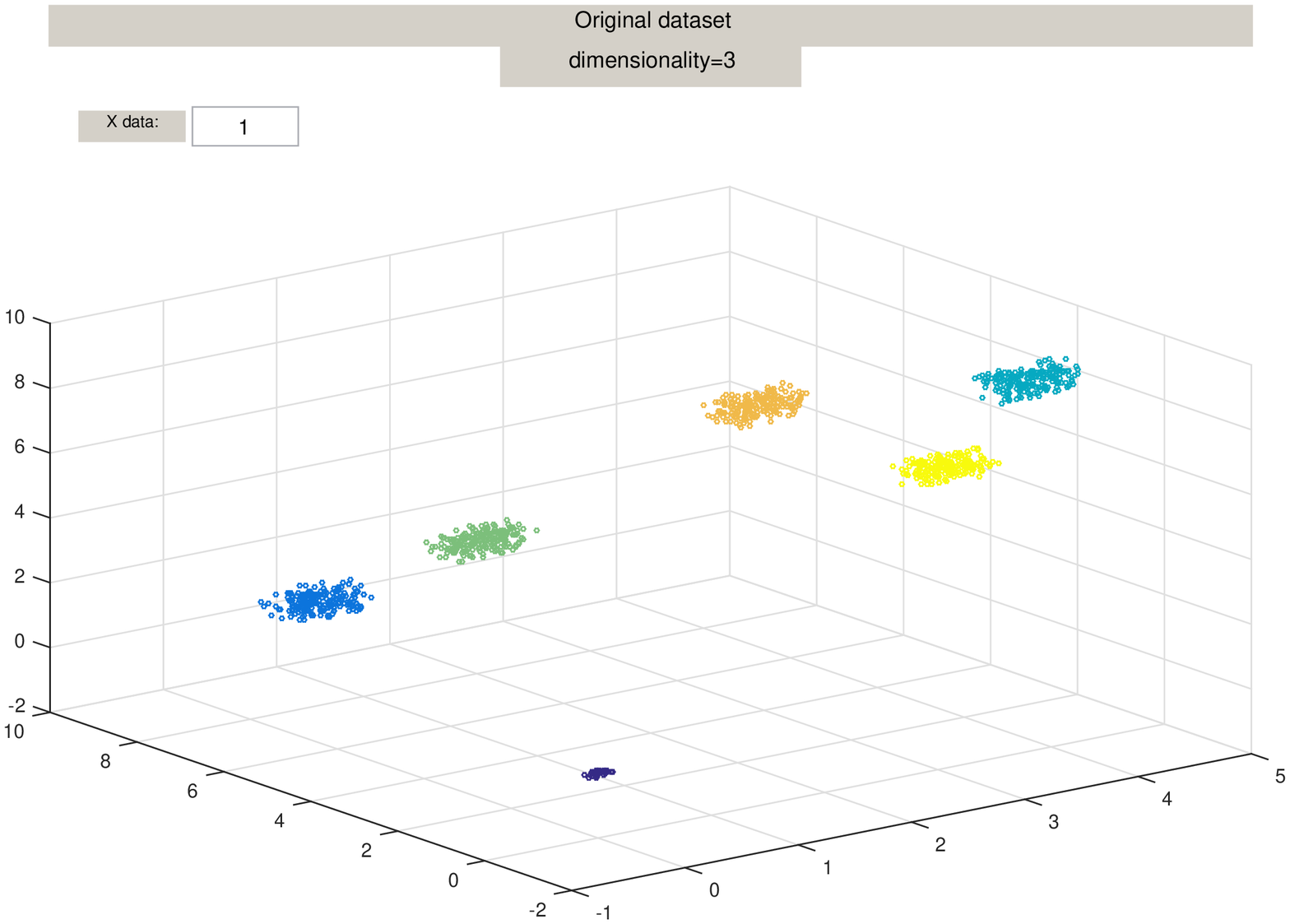}}\hfill
			\subfloat[Mapped data using Diffusion Maps]{%
				\includegraphics[width=.4\textwidth]{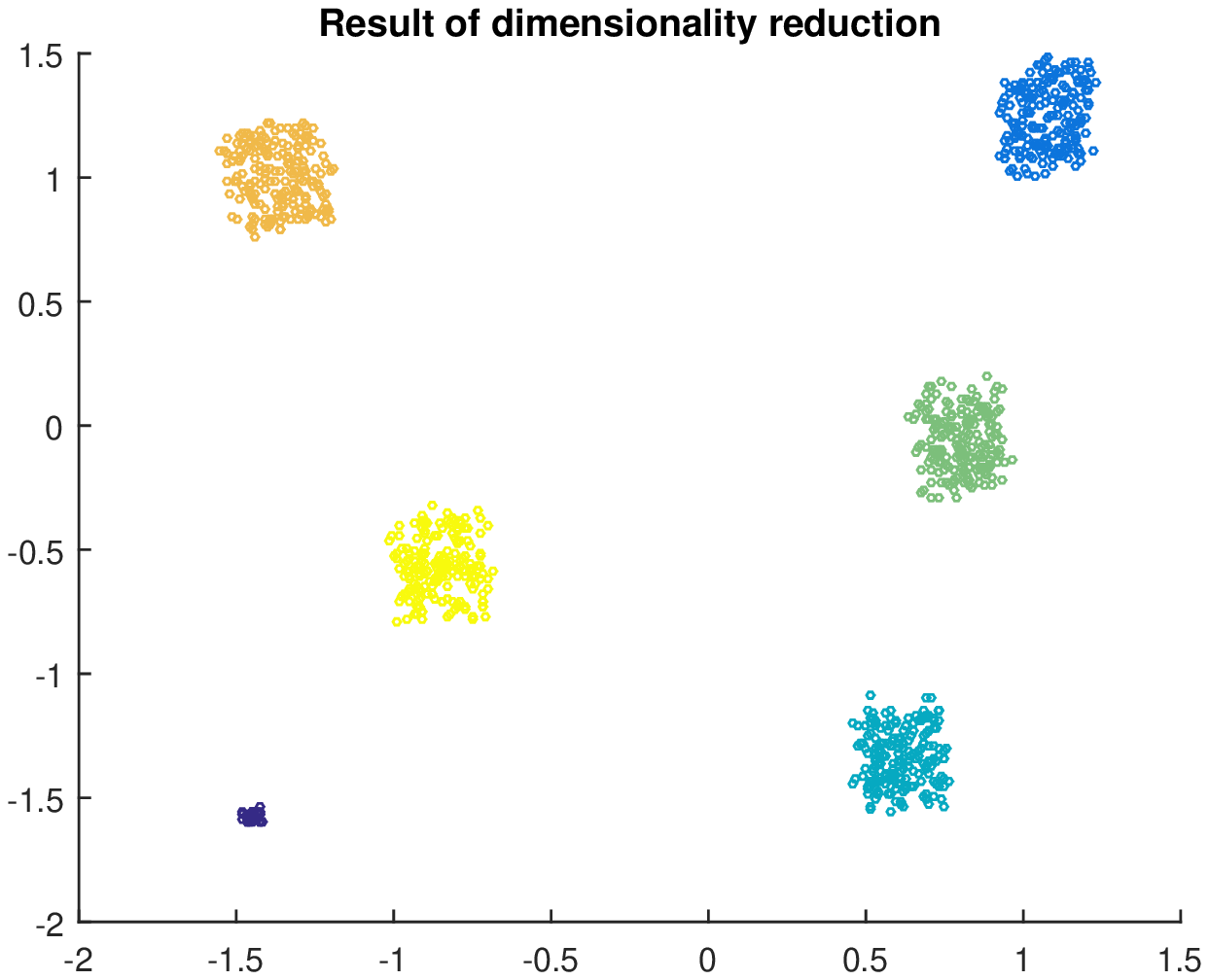}}
			\caption{The mapping of a 3D cluster using Diffusion Maps from the Matlab package drtoolbox https://lvdmaaten.github.io/drtoolbox/}\label{NLEmb}
		\end{figure}
	\end{center}	
	
	In particular, the main idea in such methods is to approximately preserve 
	the pairwise distances. Such a constraint is often inappropriate for 
	any embedding based preconditioner for any clustering technique where 
	one would like to concentrate the samples belonging to the same cluster and 
	separate the samples belonging to different clusters.
	
	\smallskip
	
	In this paper, we propose a study of Guedon and Vershynin's method 
	for finding an embedding with clustering purposes in mind. The essential 
	ingredient allowing to focus on clustering more than distance preserving 
	compression/visualization is to try to estimate the clustering matrix and use 
	spectral embedding on the cluster matrix instead of the 
	Laplacian matrix itself. 
	More precisely, the cluster matrix is the square matrix indexed by the data and whose entries are one if the associated data belong to the same cluster 
	and zero otherwise. This eigenvalue decomposition of this matrix 
	provides a perfect clustering procedure: its rank is exactly the 
	number of clusters and each data is associated 
	with exactly one eigenvector. Similarly to spectral clustering, 
	the eigenvectors give a meaningful embedding. Motivated by these
	considerations, it seems fairly reasonable to expect that a good approximation of the clustering matrix will also provide
	an efficient embedding, i.e. suitable for 
	clustering, via its eigenvalue decomposition. This intuition is supported by Remark 1.6 in \cite{guedon2015community} which we now quote: 
	{\em It may be convenient to view the cluster matrix as the adjacency matrix of the cluster graph, in which all vertices within each community are connected and there are no connections across the communities. This way, the semidefinite program takes a sparse
		graph as an input, and it returns an estimate of the cluster graph as an output. The effect of the program is thus to "densify" the network inside the communities and "sparsify" it across the communities.}

	Our goal here will thus be to approximate the clustering matrix efficiently, based on the knowledge 
	of the sample pairwise distances. Guedon and Vershynin proved that such 
	a good approximation could be found as the solution to a Semi-Definite 
	Programming (SDP) problem for community detection in the Stochastic 
	Block Model framework. We pursue this study here by considering the
	Gaussian Mixture Model framework.
	
	\subsection{Recent advances in clustering}
	
	Unsupervised clustering is a key problem in modern data analysis. Traditional approaches to clustering 
	are model based (e.g. Gaussian mixture models) or nonparametric. For mixture models, the algorithm of choice has long been the EM 
	algorithm by Dempster et al. \cite{Dempster1977}, see the monograph by McLachlan and Peel \cite{mclachlan2004finite} for an overview of finite mixture models. Nonparametric algorithms such as  $K$-means, $K$-means ++ and generalizations have been used extensively in computer science; 
	see Jain \cite{jain2010data} for a review. The main drawback of these standard approaches is that the minimization problems underlying the various procedures are not convex. Even worse, the log-likelihood function of e.g. Gaussian mixture model exhibits degenerate behavior, see Biernacki and Chr\'etien 
	\cite{biernacki2003degeneracy}.  As a result, one can never certify that such algorithms have converged to an interesting stationary point and the popularity of such methods seems to be based on their satisfactory average practical performance.
	
	\smallskip 
	
	Recently some convex minimization based methods have been proposed in the literature. A nice method using ideas similar to the LASSO is ClusterPath \cite{hocking2011clusterpath}. This very interesting and efficient method has been studied and extended in \cite{tan2015statistical}, \cite{radchenko2014consistent} and \cite{wang2016sparse}. One of the main drawbacks of this approach is the lack of a robust rule for the choice of the parameters governing the procedure although they seem to be reasonably easy to tune in practice. A closely related approach is 
	\cite{chen2015convex}.
	
	Recently, very interesting results have appeared for the closely related problem of community detection based on the stochastic block model, see Abbe et al. \cite{abbe2014exact}, Heimlicher et al. \cite{heimlicher2012community} and Mossel et al. \cite{mossel2012stochastic}. In this model, a random graph is constructed by partitioning the set of vertices $V$ into $K$ clusters $\mathcal{C}_1,\ldots,\mathcal{C}_K$ and by setting an  edge between  vertices $v$ and $v'$ with probability $p_{kk'}$ if $v\in\mathcal{C}_{k}$ and $v'\in\mathcal{C}_{k'}$. All edges are independent  and the probabilities of edges depends only on the clusters structure. It is assumed that this probability is larger within clusters, i.e.
	\begin{equation}\label{eq:p>qGuedon}
		p=\min_{1\leq k\leq K} p_{kk} > \max_{1\leq k \neq k'\leq K} p_{kk'}=q.
	\end{equation}
	This corresponds to the intuitive notion of cluster in graph theory where clusters have a higher edge density. Gu\'edon and Vershynin \cite{guedon2015community} proved that the problem of recovering the clusters from the random graph can be addressed via Semi-Definite Programming (SDP) with an explicit control of the error rate. Although not explicitly studied in their paper, the SDP can be solved efficiently thanks to a general theory, see Boyd and Vandenberghe \cite{boyd2004convex}. 
	
	\smallskip
	
	\subsection{More on the Gaussian Cluster Model}
	The mathematical framework is the following. We assume that we observe a data set $x_1,\ldots,x_n$ $\in\mathbb{R}^d$ over a population of size $n$. The population is partitioned into $K$ clusters $\mathcal{C}_1,\ldots,\mathcal{C}_K$ of size $n_1,\ldots,n_K$ respectively, i.e. $n=n_1+\cdots+n_K$. 
	We assume the standard Gaussian Cluster Model for the data: the observations $x_i$ are independent with
	\begin{equation}\label{eq:genmodel}
		x_i\sim \mathcal N(\mu_k,\Sigma_k) \quad \mbox{ if } i\in \mathcal{C}_k
	\end{equation}
	with $\mu_k\in\mathbb{R}^d$ the cluster mean and $\Sigma_k\in\mathbb{R}^{d\times d}$ the cluster covariance matrix.
	The clustering problem aims at recovering the clusters $\mathcal{C}_k$, $1\leq k\leq K$, based on the data $x_i$, $1\leq i\leq n$, only.
	For each $i=1,\ldots,n$, we will denote by $k_i$ the index of the cluster to which $i$ belongs. The notation 
	$i\sim j$ will mean that $i$ and $j$ belong to the same cluster.

	\smallskip
	
	This slightly differs from the usual setting for Gaussian unmixing. One usually assume that the data set is made of independent observations from the mixture of Gaussian distributions
	\[
	\sum_{k=1}^K \pi_k \ \mathcal N(\mu_k,\Sigma_k)
	\]
	where the vector $(\pi_k)_{1\leq k\leq K}$ gives the mixture distribution. Then the cluster sizes $(n_1,\ldots,n_k)$ are random with  multinomial distribution of size $n$ and probability parameters  $(\pi_1,\ldots,\pi_K)$. Given all the parameters of the Gaussian mixture, the probability that observation $x_i$ belongs to cluster $\mathcal C_k$ is given by
	\[
	p_k(x_i)=\frac{\pi_k \ p (x_i | \mu_k, \Sigma_k)}{\sum_{l=1}^K \pi_l \ p (x_i | \mu_l, \Sigma_l)},
	\]
	with $p(\cdot|\mu,\Sigma)$ the Gaussian distribution function. Maximizing these probabilities results in a partition of the space $\mathbb{R}^d$ into different regions $R_1,\ldots,R_K$ given by
	\[
	R_k=\left\{x\in\mathbb{R}^d; \ p (x | \mu_k, \Sigma_k)\geq p (x | \mu_l, \Sigma_l) \mbox{ for all } l\neq k\right\},\quad k=1,\ldots,K.
	\]
	The probability that an observation is misclassified is then given by
	\[
	\sum_{k=1}^K \pi_k \int 1_{x\notin R_k}p (x| \mu, \Sigma)\,\mathrm{d}x.
	\]
	Of course in practice the parameters $\pi_i,\mu_i,\Sigma_i$, $1\leq i\leq K$, are unknown and have to be estimated. The most popular approach is based on maximum likelihood estimation via the EM algorithm \cite{Dempster1977} and its variant like CEM, see C\'eleux and Govaert \cite{Celeux92}. The
	likelihood 
	\[
	L\left((\pi_k,\mu_k,\Sigma_k)_{1\leq k\leq K}\right)=\sum_{i=1}^n \sum_{k=1}^K \pi_k \: p (x_i | \mu_k, \Sigma_k)
	\] 
	may behave quite badly and exhibit degenerate behavior, making optimization via EM 
	not always reliable, see \cite{biernacki2003degeneracy}. 
	
	\smallskip
	
	Another viewpoint on the results from the present paper is to propose a low dimensional preconditioner for the Gaussian Mixture estimation problem. 
	
	\subsection{Our contribution}
	Firstly, we propose an extension of the analysis of Gu\'edon and Vershynin to the problem of low dimensional embedding via the estimation of the cluster matrix with  Gaussian clustering in mind. We provide in particular a theoretical upper bound for the misclassification rate. This adaptation is non trivial because, unlike the stochastic block model, the affinity matrix associated to Gaussian clustering does not have independent entries. Thus we need to introduce concentration inequalities for Gaussian measures, see e.g. the monograph by Boucheron et al. \cite{Boucheron2013}. Secondly, we propose a simple and scalable algorithm to solve the Semi-Definite Program based on eigenvalue optimization in the spirit of \cite{helmberg2000spectral}. Thirdly, we suggest 
	a practical way of choosing the unknown parameter $\lambda$ in the Guedon Vershynin relaxation.

	\subsection{Structure of the paper}
	The paper is organized as follows. Section \ref{sec:proof} is devoted to the proof of Proposition \ref{prop} and Theorem \ref{theo1}.  Section \ref{sec:formulae} provides explicit formulas for the expected affinity matrix $\bar A$ in the case when $f$ is the Gaussian affinity function \eqref{eq:gaussaffinity}. An efficient algorithm is described in Section \ref{sec:solving}, as well as a practical method 
	for selecting the unkown parameter.

	\section{Main results}

	\subsection{Motivation for the cluster matrix estimation approach}
	The main idea is to use the fact that the cluster matrix $\bar Z$ is a very special matrix. Indeed, if we denote by $\mathcal C_1$, \ldots, $\mathcal C_K$ the index set of each cluster, we can write $Z$ as follows: 
	\begin{align*} 
		Z & = \sum_{k=1}^K \ 1_{\mathcal{C}_k} 1_{\mathcal{C}_k}^t 
	\end{align*} 
	and thus, we conclude that 
	\begin{itemize}
		\item the rank of $\bar Z$ is $K$ 
		\item the eigenvalues of $\bar Z$ are $\sqrt{\vert \mathcal C_1\vert}$, \ldots, $\sqrt{\vert \mathcal C_K \vert}$
		\item the eigenvectors of $\bar Z$ are $1/\sqrt{\vert \mathcal C_1 \vert} \ 1_{\mathcal C_1}$, \ldots, $1/\sqrt{\vert \mathcal C_K\vert} \ 1_{\mathcal C_k}$. 
	\end{itemize}
	In the sequel, we will assume that the cluster sizes are all different. Thus, all nonzero 
	eigenvalues have multiplicity equal to one. 
	
	Based on the cluster matrix $\bar Z$, clustering is very easy: the label of each sample point $x_i$ 
	is the index of the only eigenvector whose $i^{th}$ component is non zero. Notice 
	that the $i^{th}$ component of all other eigenvectors are equal to zero. 
	
	The estimate $\widehat{Z}$ of the matrix $\bar Z$ can be used in practice to 
	embed the data into the space $\mathbb R^{\widehat{K}}$ by associating 
	each data $x_i$ to the vector consisting of the $i^{th}$ coordinate of the $\widehat{K}$ first eigenvectors of $\widehat{Z}$. Given this embedding, 
	if we can prove that $\widehat{Z}$ accurately estimates $\bar Z$, one can then 
	apply any clustering method of choice to recover the clustering pattern of the original data. The next section gives a method for computing 
	an estimator $\widehat{Z}$ of $\bar Z$.

	\subsection{Guedon and Vershynin's Semi-Definite Program for Gaussian clusters}
	We now turn to the estimation of the cluster matrix using Guedon and Vershynin's Semi-Definite Programming based approach. 
	Whereas Vershynin and Gu\'edon \cite{guedon2015community} were interested 
	in analyzing the Stochastic Block Model for community detection, we propose
	a study of the Gaussian Cluster Model and therefore prove that their 
	approach has a great potential applicability in embedding of general 
	data sets beyond the graphical model setting. 
	
	\smallskip
	
	Based on the data set $x_1,\ldots,x_n$, we construct an affinity matrix $A$ by 
	\begin{equation}\label{eq:defA}
		A=\big(f(\|x_{i}-x_{j}\|_2) \big)_{1\leq i,j\leq n} 
	\end{equation}
	where $\|\cdot\|_2$ denotes the Euclidean norm on $\mathbb{R}^d$ and $f:[0,+\infty)\to [0,1]$ an affinity function. A popular choice is the Gaussian affinity
	\begin{equation}\label{eq:gaussaffinity}
		f(h)=e^{-(h/h_0)^2},\quad h \geq 0,
	\end{equation}
	and other possibilities are
	\[
	f(h)=e^{-(h/h_0)^a}, \quad f(h)=(1+(h/h_0))^{-a}, \quad f(h)=(1+e^{h/h_0})^{-a}\quad \cdots
	\]
	
	\smallskip
	Before stating the Semi-Definite Program, we introduce some matrix notations. The  usual scalar product between matrices $A,B\in \mathbb{R}^{n\times n}$ is denoted by $\langle A,B\rangle=\sum_{1\leq i,j\leq n}A_{ij}B_{ij}$. The notations $1_n\in \mathbb{R}^{n}$ and $1_{n\times n}\in \mathbb{R}^{n\times n}$ stand for the vector and matrices with all entries equal to 1. For a symmetric matrix $Z\in\mathbb{R}^{n\times n}$, the notation $Z\succeq 0$ means that $Z$ the quadratic form associated to $Z$ is non-negative while the notation $Z\geq 0$ means that all the entries of $Z$ are non-negative. 
	
	\smallskip
	With these notations, the Semi-Definite Program writes
	\begin{equation}\label{SDP}
		\mbox{maximize }\langle A,Z\rangle \quad \mbox{subject to} \quad Z\in\mathcal{M}_{opt}
	\end{equation}
	with $\mathcal{M}_{opt}$ the set of symmetric matrices $Z\in\mathbb{R}^{n\times n}$ such that
	\begin{equation}\label{Mopt}
		\left\{\begin{array}{l}
			Z\succeq 0  \\
			Z\geq 0 \\
			\mathrm{diag}(Z)= 1_n \\
			\langle Z,1_{n\times n}\rangle= \lambda_0 
		\end{array}
		\right..
	\end{equation}
	
	Let us provide some intuitions for motivating the SDP problem \eqref{SDP}. Note that each $Z\in\mathcal{M}_{opt}$ has entries in $[0,1]$ with constant sum equal to $\lambda_0$. The SDP procedure will distribute the mass $\lambda_0$ and assign more mass to entries $Z_{ij}$ corresponding to large values of the affinity $A_{ij}=f(\|x_j-x_i\|_2)$, i.e. pairs of close points $x_i,x_j$. This mass distribution must respect symmetry and the constraint $Z\succeq 0$. For the analysis of the procedure, the main idea is that we want the solution $\widehat Z$ to be an approximation of $\bar Z$, the cluster matrix defined by
	\begin{equation}\label{barZ}
		\bar Z_{i,j}=\left\{\begin{array}{ll} 1 & \mbox{ if } i \mbox{ and  }j \mbox{ are in the same cluster} \\ 
			\\
			0 & \mbox{otherwise}\end{array} \right..
	\end{equation}
	The cluster matrix has values in $\{0,1\}$ and belongs to $\mathcal{M}_{opt}$ for $\lambda_0=\sum_{k=1}^K n_k^2$ given by the cluster sizes. In practice, $\lambda_0$ is unknown and should be estimated, see some comment in section 5.1.5. Under some natural assumption (see Equation \eqref{eq:p>q} below), the cluster matrix is the solution of the alternative SDP problem
	\begin{equation}\label{SDP2}
		\mbox{maximize } \quad \langle \bar A,Z\rangle \quad \mbox{subject to} \quad Z\in\mathcal{M}_{opt}
	\end{equation}
	where $\bar A$ denotes the expected affinity matrix defined by
	\begin{equation}\label{barA}
		\bar A= \Big(\mathrm{E}f(\|x_i-x_j\|_2)\Big)_{1\leq i,j\leq n}.
	\end{equation}
	The affinity matrix $A$ is a very noisy observation of $\bar A$ but concentration arguments together with Grothendieck theorem allow to prove  that $A\approx \bar A$ in the sense of the $\ell^{\infty}$--$\ell^1$ norm. In turn, this implies $\widehat Z \approx \bar Z$ (in the sense of $\ell^1$ norm) so that the SDP program \eqref{SDP} provides a good approximation $\widehat Z$ of the cluster matrix.
	
	\subsection{Main results}
	Our main result provides a non asymptotic upper bound for the probability that $\widehat Z$ differs from $\bar Z$ in $L^1$ distance.
	\begin{theorem}\label{theo1}
		Consider the Gaussian Cluster Model \eqref{eq:genmodel}. Assume that the affinity function $f$ is $\ell$-Lipschitz and furthermore that
		\begin{equation}\label{eq:p>q}
			p=\inf_{i\sim j} \bar A_{i,j} > q=\sup_{i\nsim j} \bar A_{i,j}.
		\end{equation}
		Let 
		\begin{align*}
			t_0 & =8\sqrt{2 \log 2}K_G \sigma \ell/(p-q).
		\end{align*}
		
		Then, for all $t > t_0=8\sqrt{2 \log 2}K_G \sigma \ell/(p-q)$,
		\begin{equation}\label{eq:theo1}
			\mathrm{P}\left( \left\|\widehat Z-\bar Z\right\|_1>n^2 t \right) \leq 2 \exp\left(-\left(\frac{t-t_0}{c}\right)^2 n \right), \quad  c=\frac{16\sqrt 2 K_G \ell \sigma}{p-q}
		\end{equation}
		where $K_G\le 1.8$ denotes the Grothendieck constant and $\sigma^2=\frac{1}{n}\sum_{k=1}^K n_k \rho(\Sigma_k)$ with $\rho(\Sigma_k)$ the largest eigenvalue of the covariance matrix $\Sigma_k$. Moreover, there exists a subset $\tau
		\subset \{1,\ldots,n\}$ with $|\tau|\ge \frac{n}2$ such that all $t > t_0$,
		\begin{equation}\label{eq:theo2}
			\mathrm{P}\left( \left\|
			\left(\widehat Z-\bar Z\right)_{\tau \times \tau}\right\|_1>n t \right) \leq 2 \exp\left(-\left(\frac{t-t_0}{c}\right)^2 n \right), \quad  c=\frac{16\sqrt 2 K_G \ell \sigma}{p-q}.
		\end{equation}
	\end{theorem}
	Condition \eqref{eq:p>q} ensures that the affinity matrix $\bar A$ allows to identify the clusters and appears also in \cite{guedon2015community}, see Eq. \eqref{eq:p>qGuedon}. In the case of the Gaussian affinity function \eqref{eq:gaussaffinity}, we provide in Section \ref{sec:formulae}  explicit formulae for the expected affinity matrix that can be used to check condition \eqref{eq:p>q}.
	
	Theorem \ref{theo1} has a simple consequence in terms of estimation error rate. After computing $\widehat Z$, it is natural to estimate the cluster graph $\bar Z$ by a random graph obtained  by putting an edge between vertices $i$ and $j$ if $\widehat Z_{i,j}>1/2$ and no edge otherwise. Then the proportion $\pi_n$ of errors in the prediction of the $n(n-1)/2$ edges is given by 
	\begin{eqnarray*}
		\pi_n&:=&\frac{2}{n(n-1)} \sum_{1\leq i<j\leq n} | 1_{\{\widehat Z_{ij}>1/2\}}-\bar Z_{ij}|\\
		&\leq& \frac{2}{n(n-1)} \sum_{1\leq i<j\leq n} 2\ | \widehat Z_{ij}-\bar Z_{ij}| = \frac{2}{n(n-1)} \left\|\widehat Z-\bar Z\right\|_1 .
	\end{eqnarray*}
	The following corollary provides a simple bound for the asymptotic error.
	\begin{corollary}\label{cor1}
		We have almost surely
		\[
		\limsup_{n\to\infty} n^{-2}\left\|\widehat Z-\bar Z\right\|_1 \leq  t_0= \frac{8\sqrt{2 \log 2}K_G \sigma \ell}{p-q} .
		\]
	\end{corollary}
	In the case when the cluster means are pairwise different and fixed while the cluster variances converge to $0$, i.e. $\sigma \to 0$, it is easily seen that the right hand side of the above inequality behaves as $O(\sigma)$ so that the error rate converges to $0$. This reflects the fact that when all clusters concentrates around their means, clustering becomes trivial. 
	
	\begin{remark}
		Theorem \ref{theo1} assumes that  $\lambda_0$ is known. It is worth noting that $\lambda_0$ corresponds to  the number of edges in the cluster graph and that we can derive from the proof of Theorem \ref{theo1} how the algorithm behaves  when the cluster sizes are unknown, i.e. when the unknown parameter  $\lambda_0$ is replaced by a different value $\lambda$.  The intuition is given in Remark 1.6 of Gu\'edon and Vershynin: if $\lambda<\lambda_0$, the solution $\widehat Z$ will estimate a certain subgraph of the cluster graph with at most $\lambda_0-\lambda$ missing edges; if $\lambda>\lambda_0$, the solution  $\widehat Z$ will estimate a certain supergraph of the cluster graph with at most $\lambda_0-\lambda$ extra-edges.
	\end{remark}
	
	While our proof of Theorem \ref{theo1} follows the ideas from Vershynin and Gu\'edon \cite{guedon2015community}, we need to introduce new tools to justify the approximation $A\approx \bar A$ in $\ell^{\infty}\rightarrow \ell^{1}$-norm. Indeed, unlike in the stochastic block model, the entries of the affinity matrix \eqref{eq:defA} are not independent. We use Gaussian concentration measure arguments to obtain the following concentration inequality. The $\ell^\infty$--$\ell^1$ norm of a matrix $M\in\mathbb{R}^{n\times n}$ is defined by
	\begin{equation}\label{eq:inftyto1}
		\|M\|_{\infty\to 1}=\sup_{\|u\|_\infty\leq 1} \|Au\|_1=\max_{u, \ v\in\{-1,1\}^n} \sum_{i,j=1}^n u_iv_j M_{i,j}.
	\end{equation}
	\begin{proposition}\label{prop}
		Consider the Gaussian mixture model \eqref{eq:genmodel} and assume the affinity function $f$ is $\ell$-Lipschitz.
		Then, for any $t > 2 \ \sqrt{2 \log 2} \ \ell \ \sigma$,
		\begin{equation}\label{eq:concentration}
			\mathrm{P} \Big(\left\| A-\bar{A}\right\|_{\infty \rightarrow 1} > t\ n^2\Big)  \leq 2 \exp \left( -\frac{\left(t-2\sqrt{2 \log 2}\ell \sigma \right)^2}{32 \ell^2 \sigma^2} \ n\right).
		\end{equation}
	\end{proposition}

	\section{Proofs}\label{sec:proof}
	
	\subsection{Proof of Proposition \ref{prop}}
	\begin{proof}
		The concentration of the affinity matrix $A$ around its mean $\bar A$ follows from  concentration inequalities for Lipschitz function of independent standard Gaussian variables, see Appendix \ref{app:logsob}. From definition \eqref{eq:inftyto1} 
		\begin{equation}\label{defA}
			\|A-\bar A\|_{\infty\to 1}=\max_{u,v\in\{-1,1\}^n} F_{uv} \quad \mbox{with}\quad  F_{uv}=\sum_{i,j=1}^n u_iv_j(A_{i,j}-\bar A_{i,j}).
		\end{equation}
		We introduce the standardized observations: if $x_i$ is in cluster $\mathcal{C}_{k_i}$, i.e. $x_i\sim \mathcal{N}(\mu_{k_i},\Sigma_{k_i})$, then  $y_i=\Sigma_{k_i}^{-1/2}(x_i-\mu_{k_i})$, $1\leq i\leq n$ are independent identically distributed random variables with standard Gaussian distribution. In view of definition \eqref{defA}, the random variables $F_{uv}$ can be expressed in terms of the standardized observations 
		\begin{eqnarray*}
			F_{uv}(y_1,\ldots,y_n)&=&2\sum_{1\leq i<j\leq n} u_iv_j \left[f\left(\left\|\Sigma_{k_j}^{1/2}y_j-\Sigma_{k_i}^{1/2}y_i+\mu_{k_j}-\mu_{k_i}\right\|_2\right)-\bar A_{i,j})\right].
		\end{eqnarray*}
		We prove next that the function $F_{uv}:\mathbb{R}^{p\times n}\to\mathbb{R}$ is $L$-Lipschitz with 
		$L=2  \ell\sigma n^{3/2}$. Indeed, for $(y_1,\ldots,y_n), (y'_1,\ldots,y'_n) \in \mathbb{R}^{p\times n}$, we have
		\begin{eqnarray*}
			\left|F_{uv}(y_1,\ldots,z_n)-F_{uv}(y_1,\ldots,z_n)\right|
			&\leq &\ell\sum_{1\leq i\neq j\leq n} \|x_i-x'_i\|_2+\|x_j-x'_j\|_2 \\
			&= & 2(n-1)\ell\sum_{i=1}^n \|\Sigma_{k_i}^{1/2}(z_i-z'_i)\|_2\\
			&\leq & 2n\ell\sum_{i=1}^n \rho(\Sigma_{k_i})^{1/2} \|z_i-z'_i\|_2\\
			&\leq & 2\ell\sigma n^{3/2}\|(z_1,\ldots,z_n)-(z_1',\ldots,z_n') \|_2.
		\end{eqnarray*}
		In the first inequality, we use the fact that $f$ is $\ell$-Lipschitz. The second inequality relies on the fact that all the eigenvalues of $\Sigma_{k_i}^{1/2}$ are smaller that $\rho(\Sigma_{k_i})$. The last inequality relies on Cauchy-Schwartz inequality and on the definition $\sigma^2=\frac{1}{n} \sum_{i=1}^n \max_{1\leq k\leq K} \rho(\Sigma_{k_i})$.
		
		\smallskip
		Thanks to this Lipschitz property,  the Tsirelson-Ibragimov-Sudakov inequality (Theorem \ref{thm:lip} in the Appendix) implies
		\[
		\mathrm{E}\left[\exp(\theta F_{uv})\right] \leq \exp\left( L^2\theta^2/2\right)\quad \mbox{for all } \ \theta\in\mathbb{R}
		\]
		and we deduce from Theorem \ref{thm:Dudley} that
		\begin{eqnarray*}
			\mathrm{E}\left[ \|A-\bar A\|_{\infty\to 1}\right]
			&=&\mathrm{E}\left[ \max_{u,v\in\{-1,1\}^n} F_{uv} \right] \\
			\\
			&\leq &   \sqrt{2 L^2 \log 2^n} = 2\sqrt{ 2\log 2} \ell \sigma n^2 .
		\end{eqnarray*}
		On the other hand, the function  $\max_{u,v\in\{-1,1\}^n} F_{uv}$ is also $L$-Lipschitz and Theorem \ref{thm:lip} implies
		\begin{eqnarray*}
			\mathrm{P} \left( | \|A-\bar A\|_{\infty\to 1} - \mathrm{E} \|A-\bar A\|_{\infty\to 1}  | > t \right)
			&=& \mathrm{P} \left( | \max_{u,v\in\{-1,1\}^n} F_{uv} - \mathrm{E} \max_{u,v\in\{-1,1\}^n} F_{uv}  | > t \right)\\
			&\leq & 2 \exp \left(-\frac{t^2}{8 L^2} \right) .
		\end{eqnarray*}
		Combining these different estimates, we obtain  for  $t > 2\sqrt{2 \log 2}\ell \sigma$,
		\begin{eqnarray*}
			&&\mathrm{P} (\left\| A-\bar{A}\right\|_{\infty \rightarrow 1} > t n^2) \\
			&\leq& \mathrm{P} \left(\left|\left\| A-\bar{A}\right\|_{\infty \rightarrow 1} - \mathrm{E} \left\| A-\bar{A}\right\|_{\infty \rightarrow 1} \right|
			> (t-2\sqrt{2 \log 2} \ell \sigma) n^2 )\right) \\
			&\leq& 2 \exp \left( -\frac{\left(t-2\sqrt{2 \log 2}\ell \sigma \right)^2}{32 \ell^2 \sigma^2}n\right).
		\end{eqnarray*}
	\end{proof}

	\subsection{Proof of Theorem \ref{theo1}}
	
	The proof follows the same lines as in Gu\'edon and Vershynin \cite{guedon2015community} and we provide the main ideas for the sake of completeness. The proof is divided into 4 steps.
	
	\subsubsection{Step 1}  We show $\bar Z$ solves the SDP problem \eqref{SDP2}.
	This corresponds to Lemma 7.1 in \cite{guedon2015community}. This is proved simply as follows. Since $\mathcal{M}_{opt}\subset [0,1]^{n\times n}$, we transform the SDP problem \eqref{SDP2} into the simpler problem
	\[
	\textrm{maximize} \quad \langle Z,\bar A\rangle \quad \mbox{subject to the constraints}\ Z\in[0,1]^{n\times n} \mbox{ and } \langle Z,1_{n\times n}\rangle=\lambda_0.
	\]   
	In order to solve this second problem, the mass $\lambda_0$ has to be assigned to the $\lambda_0$ entries where $\bar A_{ij}$ is maximal. Thanks to \eqref{eq:p>q}, this corresponds exactly to the cluster matrix $\bar Z$. One can then check a posteriori that $\bar Z\in\mathcal{M}_{opt}$ so that in fact the original SDP problem \eqref{SDP2} has been solved.
	
	\subsubsection{Step 2} We now prove that
	\begin{equation}\label{eq:step2}
		\langle \bar A,\bar Z \rangle-2 K_G \| A-\bar A\|_{\infty\to 1}\leq \langle\bar A,\widehat Z\rangle \leq \langle \bar A,\bar Z\rangle
	\end{equation}
	with $K_G$ denoting Grothendieck's constant.\\
	The upper bound follows directly from step 1. For the lower bound, we use the definition of $\widehat Z$ as a maximizer and write 
	\begin{eqnarray*}
		\langle\bar A,\widehat Z\rangle &=& \langle A,\widehat Z\rangle + \langle\bar A-A,\widehat Z\rangle \\
		&\geq & \langle A,\bar Z\rangle - \langle A-\bar A,\widehat Z\rangle\\
		&=& \langle \bar A,\bar Z\rangle + \langle A-\bar A,\bar Z\rangle - \langle A-\bar A,\widehat Z\rangle.
	\end{eqnarray*}
	Grothendieck's inequality   implies that for every $Z\in\mathcal{M}_{opt}$,
	\[
	\left| \langle A-\bar A,\widehat Z\rangle \right| \leq K_G \| A-\bar A\|_{\infty\to 1}.
	\]
	See Theorem \ref{theoG} and Lemma \ref{lemG} in the Appendix. Using this, we get
	\begin{align}
		2K_G \| A-\bar A\|_{\infty\to 1} & \ge \langle\bar A,\widehat Z-\bar Z\rangle.
		\label{grothineg}
	\end{align}
	
	\subsubsection{Step 3} We show that for every $Z\in\mathcal{M}_{opt}$,
	\begin{equation}\label{eq:step3}
		\langle \bar A,\bar Z-Z\rangle \geq \frac{p-q}{2} \left\|\bar Z-Z\right\|_1.
	\end{equation}
	This corresponds to  Lemma 7.2 in \cite{guedon2015community} and shows that the expected objective function distinguishes points.
	Introducing the set 
	\begin{align}
		\mathrm{In}=\cup_{k=1}^K \mathcal{C}_k\times \mathcal{C}_k
	\end{align}of edges within clusters and the set
	\begin{align}
		\mathrm{Out}=\{1,\ldots,n\}^2\setminus \mathrm{In}
	\end{align} 
	of edges across clusters, we decompose the scalar product
	\[
	\langle \bar A,\bar Z-Z\rangle=\sum_{(i,j)\in\mathrm{In}} \bar A_{ij}(\bar Z_{ij}-Z_{ij})-\sum_{(i,j)\in\mathrm{Out}} \bar A_{ij}(Z_{ij}-\bar Z_{ij}).
	\]
	Note that the definition of the cluster matrix \eqref{barZ} implies that $\bar Z_{ij}-Z_{ij}\geq 0$ if $(i,j)\in\mathrm{In}$ and $\bar Z_{ij}-Z_{ij}\leq 0$ if $(i,j)\in\mathrm{Out}$. This together with condition \eqref{eq:p>q}  implies
	\begin{align*}
		\langle \bar A,\bar Z-Z\rangle\geq p \sum_{(i,j)\in\mathrm{In}} (\bar Z-Z)_{ij}-q\sum_{(i,j)\in\mathrm{Out}} (Z-\bar Z)_{ij}.
	\end{align*}
	Introduce $S_{\mathrm{In}}=\sum_{(i,j)\in\mathrm{In}} (\bar Z-Z)_{ij}$ and $S_{\mathrm{Out}}=\sum_{(i,j)\in\mathrm{Out}} (\bar Z-Z)_{ij}$. Since 
	$\langle \bar Z, 1_{n\times n}\rangle=\langle  Z, 1_{n\times n}\rangle=\lambda_0$, we have $S_{\mathrm{In}}-S_{\mathrm{Out}}=0$. On the other
	hand $S_{\mathrm{In}}+S_{\mathrm{Out}}=\left\|\bar Z-Z\right\|_1$. We deduce exact expressions for $S_{\mathrm{In}}$ and $S_{\mathrm{Out}}$ and the lower bound
	\begin{align}
		\langle \bar A,\bar Z-Z\rangle\geq \frac{p-q}{2} \left\|\bar Z-Z\right\|_1.
		\label{l1bnd}
	\end{align}
	
	Combining \eqref{grothineg}, \eqref{l1bnd} and \eqref{grothfact}, we obtain 
	\begin{align}
		\left\|\bar Z -\widehat Z\right\|_1 \leq  \frac1{n} \frac{4K_G}{p-q} \ 
		\|A-\bar A\|_{\infty\to 1}.
		\label{keyeq}
	\end{align}
	
	\subsubsection{Proof of \eqref{eq:theo1}} 
	Using \eqref{keyeq}, we may deduce that 
	\begin{align*}
		\mathrm{P} \left(\left\| Z-\bar{Z}\right\|_{1} > t\ n^2 \right) \leq \mathrm{P}  \left(\left\| A-\bar{A}\right\|_{\infty \rightarrow 1} > t \ \frac{p-q}{4 K_G} \ n^2\right).
	\end{align*}
	and \eqref{eq:theo1} follows then directly from \eqref{eq:concentration}.
	
	\subsubsection{Proof of \eqref{eq:theo2}} 
	For every matrix $H \in \mathbb R^{n\times n}$, we have that 
	\begin{align}
		\left\|G\right\|_1 & \ge \left\|G\right\|_{\infty \rightarrow 1}.
		\label{linfty1gel1}
	\end{align}
	Using Proposition 5.2 in \cite{Tropp:SODA09}, we obtain that 
	there exists a subset $\tau \in \{1,\ldots,n\}$ such that 
	$|\tau|\ge \frac{n}2$ and 
	\begin{align*}
		\Vert H_{\tau \times \tau} \Vert_1 & \le \frac{2 K_G}{n} \ \Vert H\Vert_{\infty \rightarrow 1}.
	\end{align*}
	Therefore, taking $H=\bar Z -\widehat Z$, we obtain that 
	\begin{align}
		\Big\Vert \left(\bar Z -\widehat Z\right)_{\tau \times \tau} \Big\Vert_1 & \le \frac{2 K_G}{n} \ \Big\Vert \left(\bar Z -\hat Z\right)_{\tau \times \tau} \Big\Vert_{\infty \rightarrow 1}.
	\end{align}
	Using \eqref{linfty1gel1}, we obtain 
	\begin{align}
		\Big\Vert \left(\bar Z -\widehat Z\right)_{\tau \times \tau} \Big\Vert_1 & \le \frac{2 K_G}{n} \ \Big\Vert \left(\bar Z -\hat Z\right)_{\tau \times \tau} \Big\Vert_{1}.
		\label{grothfact}
	\end{align}
	Combining this last equation with \eqref{keyeq}, we obtain 
	\begin{align*}
		\left\|\left(\bar Z -\widehat Z\right)_{\tau \times \tau}\right\|_1 \leq  \frac1{n} \frac{8K_G^2}{p-q} \ \| A-\bar A\|_{\infty\to 1}.
	\end{align*}
	We thus may deduce that 
	\begin{align*}
		\mathrm{P} \left(\left\| \left(\bar Z -\widehat Z\right)_{\tau \times \tau}\right\|_1 > t\ n^2 \right) \leq \mathrm{P}  \left(\left\| A-\bar{A}\right\|_{\infty \rightarrow 1} > t \ \frac{p-q}{4 K_G} \ n^2\right).
	\end{align*}
	and \eqref{eq:theo2} follows then directly from \eqref{eq:concentration}.
	
	\section{Explicit formul{\ae} for the expected affinity matrix}\label{sec:formulae}
	In order to check condition \eqref{eq:p>q}, explicit formulas for the mean affinity matrix are useful. 
	\begin{proposition}
		Assume that $A$ is build using the Gaussian affinity function \eqref{eq:gaussaffinity}.
		\begin{itemize}
			\item Let $i$ and $j$  be in the same cluster $\mathcal{C}_{k}$.
			Then, 
			\[
			\bar A_{i,j}=\prod_{l=1}^d \left(1+4(\sigma_{k,l}/h_0)^2 \right)^{-1/2} 
			\]
			with $(\sigma_{k,l}^2)_{1\leq l\leq d}$ the eigenvalues of $\Sigma_k$.
			\item Let $i$ and $j$  be in different clusters $\mathcal{C}_{k}$ and $\mathcal{C}_{k'}$. Then,
			\[
			\bar A_{i,j}=\prod_{l=1}^d \exp\left(-\frac{\langle \mu_k-\mu_{k'},v_{k,k',l} \rangle^2}{h_0^2+2\sigma_{k,k',l}^2} \right) 
			\left(1+2(\sigma_{k,k',l}/h_0)^2 \right)^{-1/2} 
			\]
			with $(\sigma_{k,k',l}^2)_{1\leq l\leq d}$ and $(v_{k,k',l})_{1\leq l\leq d}$ respectively the eigenvalues and eigenvectors of 
			$\Sigma_k+\Sigma_{k'}$.
		\end{itemize}
		
	\end{proposition}
	\begin{proof}
		The proof of the proposition relies on the fact that $X_i-X_j$ is a Gaussian random vector with mean $\mu_{k_i}-\mu_{k_j}$ and variance 
		$\Sigma_{k_i}+\Sigma_{k_i}$ so that the distribution of $\|X_i-X_j\|_2^2$ is related to the noncentral $\chi^2$ distribution with $p$ degrees of freedom.
		The next Lemma provides the Laplace transform of the noncentral $\chi^2$ distribution. 
		\begin{lemma}
			Let $X\sim \mathcal{N}(\mu,\Sigma)$. If $\mu\neq 0$, we have 
			\[
			\mathrm{E}\left[e^{t \|X\|^2}\right] =\prod_{d=1}^p \left(1-2t\sigma^2_d\right)^{-1/2},\quad t\leq 0,
			\]
			with $\sigma^2_1,\ldots,\sigma^2_p$ the eigenvalues of $\Sigma$. More generally, for $\mu\neq 0$,
			\[
			\mathrm{E}\left[e^{t\|X\|^2}\right] =\prod_{d=1}^p \exp\left(\frac{\langle \mu,v_d \rangle^2t}{1-2t\sigma_d^2} \right) \left(1-2t\sigma_d^2 \right)^{-1/2} 
			\]
			with $\sigma^2_1,\ldots,\sigma^2_p$ the eigenvalues of $\Sigma$ and $v_1,\ldots,v_p$ the associated eigenvectors.
		\end{lemma}
	\end{proof}
	
	\begin{remark}
		In dimension $p=1$, we obtain the very simple formula
		\[
		\bar A_{i,j}=\exp\left(-\frac{(\mu_k-\mu_{k'})^2}{h_0^2+2(\sigma_{k}^2+\sigma_{k'}^2)} \right) \left(1+2(\sigma_{k}^2
		+\sigma_{k'}^2/h_0^2) \right)^{-1/2}. 
		\]
		In the case of $K=2$ clusters, the condition $p>q$ writes
		\[
		(\mu_2-\mu_{1})^2>\frac{1}{2}\left(h_0^2+2(\sigma_{1}^2+\sigma_{2}^2)\right)\max_{k=1,2}\log\frac{\left(1+4\sigma_{k}^2/h_0^2 \right)}{\left(1+2(\sigma_{1}^2+\sigma_{2}^2)/h_0^2 \right)}.
		\]
		When $\sigma_1=\sigma_2$, we simply need $\mu_2\neq\mu_{1}$.
	\end{remark}

	\section{Computing a solution of the SDP relaxation}
	\label{sec:solving}
	
	\subsection{The algorithm}
	\subsubsection{Further description of the constraints}
	
	The constraints that every diagonal element of $Z$ should be equal to $1$ can be written as  
	\begin{equation}
		\trace \left( C_i Z \right) = 1, \quad C_i= e_i e_i^T, \quad i=1, \ldots n. \label{eq:diag}
	\end{equation}
	The $\displaystyle\sum_{i, j=1}^n Z_{i,j} = \lambda$ constraint can be expressed as the rank-$1$ constraint
	\begin{equation}
		\trace \left( D Z \right) = \lambda , \label{eq:lambda}
	\end{equation}
	where $D$ is the all-ones matrix of size $n \times n$.
	
	\subsubsection{Helmberg and Rendl's spectral formulation}
	
	One of the nice features of the Semi-Definite Program \eqref{SDP}-\eqref{Mopt} is that it can be rewritten as an eigenvalue optimization problem. Let us adopt a Lagrangian duality approach to this problem. As in \cite{helmberg2000spectral}, we will impose the 
	reduntant constraint that the trace is constant. The Lagrange function is given by 
	\begin{align*}
		L(Z,z) & = \langle A,Z\rangle + \sum_{i=1}^n z_i \ \left(\langle C_i,Z \rangle-1\right)+
		z_{n+1} \left(\langle D,Z \rangle-\lambda\right) \\
		& = \left\langle A+ \sum_{i=1}^n z_i C_i+z_{n+1} D,Z \right\rangle -\sum_{i=1}^n z_i-\lambda z_{n+1}.
	\end{align*}
	Now the dual function is given by 
	\begin{align}
		\theta (z) & = \max_{\stackrel{Z \succeq 0}{{\rm trace}(Z)=n}} \quad L(Z,z). 
		\label{dual}
	\end{align}
	Therefore, we easily get that 
	\begin{align*}
		\theta (z) & = n \ \lambda_{\max} \left(A+ \sum_{i=1}^n z_iC_i+z_{n+1} D \right) -\sum_{i=1}^n z_i-\lambda z_{n+1}
	\end{align*}
	where $\lambda_{\max}$ is the maximum eigenvalue function. 
	Therefore, the solution to this problem is amenable to eigenvalue optimization. An important remark is that 
	a maximizer $Z^*$ in \eqref{dual} associated to a dual minimizer $z^*$ will be a solution of the original problem. It can be written as 
	\begin{align*}
		Z^* & = V^*V^{*^t}
	\end{align*}
	where $V^*$ is a matrix whose columns form a basis of the eigenspace associated with $\lambda_{\max} \left(A+ \sum_{i=1}^n z_iC_i+z_{n+1} D \right)$.
	
	\subsubsection{Using the HANSO algorithm}
	
	The maximum eigenvalue is a convex function. Let $\mathcal A$ denote the affine operator 
	\begin{align*}
		\mathcal A (z) & = A+ \sum_{i=1}^n z_iC_i+z_{n+1} D. 
	\end{align*} 
	The subdifferential of $\lambda_{\max} (\mathcal A(z))$ is given by 
	\begin{align*}
		\partial \lambda_{\max}(\mathcal A(z)) & = \mathcal A^* \left( V \mathcal Z V^t\right) 
	\end{align*}
	where $V$ is a matrix whose columns form a basis for the eigenspace associated to the maximum eigenvalue of $\mathcal A(z)$ and  
	\begin{align*} 
		\mathcal Z & = \left\{ Z \in \mathbb R^{r_{\max} \times r_{\max}} \mid Z\succeq 0 
		\textrm{ and } \textrm{trace} \left(Z\right)=1 \right\}
	\end{align*} 
	where $r_{\max}$ is the multiplicity of this maximum eigenvalue. 
	
	With these informations in hand, it is easy to minimize the dual function $\theta$. Indeed, 
	the subdifferential of $\theta$ at $z$ is simply given by 
	\begin{align*}
		\partial \theta (z) & = \mathcal A^* \left(V \mathcal Z V^t\right)-
		\left\{
		\left[ 
		\begin{array}{c}
			e \\
			\lambda 
		\end{array} 
		\right]
		\right\}.  
	\end{align*}
	The algorithm HANSO \cite{overton2009hanso} can then be used to perform the 
	actual minimization of the dual function $\theta$. A primal solution 
	can then be recovered as a maximizer in the definition 
	\eqref{dual} of the dual function. 
	
	\subsubsection{Computing the actual clustering}
	As advised in \cite{vu2011singular}, the actual clustering can be computed using a minimum spanning tree method and removing the largest edges. 
	
	\subsubsection{Choosing $\lambda$}
	Once the clusters have been identified, it is quite easy to identify the underlying Gaussian Mixture Distribution. 
	The choice of $\lambda$ can then be performed using 
	standard model selection criteria such as AIC \cite{akaike1974new}, BIC \cite{schwarz1978estimating} or ICL \cite{biernacki2000assessing}.  
	
	\section{Simulation results}
	
	In all the experiments, the parameter $h_0$ in \eqref{eq:gaussaffinity} was chosen as $$h_0=.5*\max(\textrm{diag}(X^t*X))^{1/2}.$$

	\subsection{Some interesting examples}
	
	In this section, we start with three examples where the separation properties of the Guedon-Vershynin embedding are nicely illustrated. In all three instances, the data were generated using two 2- dimensional Gaussians samples with equal size (100 samples by cluster). These examples are shown in Figure \ref{ex1}, Figure \ref{ex2} and Figure \ref{ex3} below. All these example seem to be very difficult to address for methods with guaranteed polynomial time convergence. In each case,  one observes that the clusters are well separated after the embedding and that they do not 
	look like Gaussian samples anymore. The fact of not being Gaussian does not impair the success of methods such as minimum spanning trees although such methods might work better with the example in Figure \ref{ex3} than in the example of Figure \ref{ex1} and Figure \ref{ex2}.  
	
	\begin{figure}[H]
		\centering
		\includegraphics[scale=.5]{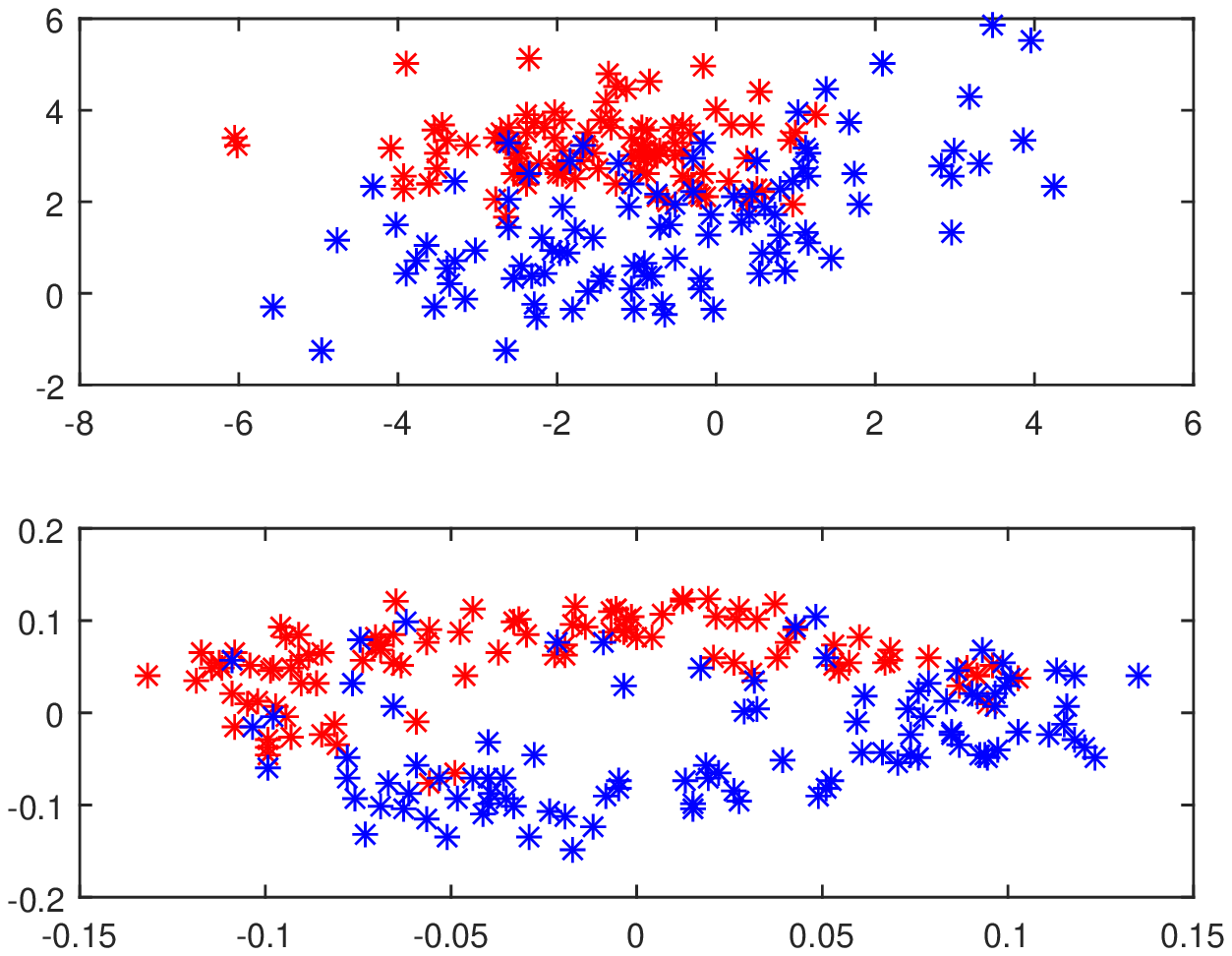}
		\caption{Example 1}
		\label{ex1}
	\end{figure}
	
	\begin{figure}[H]
		\centering
		\includegraphics[scale=.5]{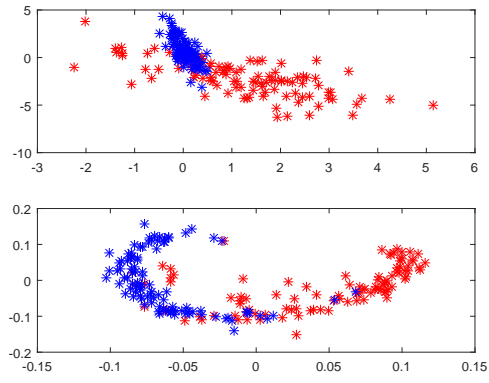}
		\caption{Example 2}
		\label{ex2}
	\end{figure}
	
	\begin{figure}[H]
		\centering
		\includegraphics[scale=.5]{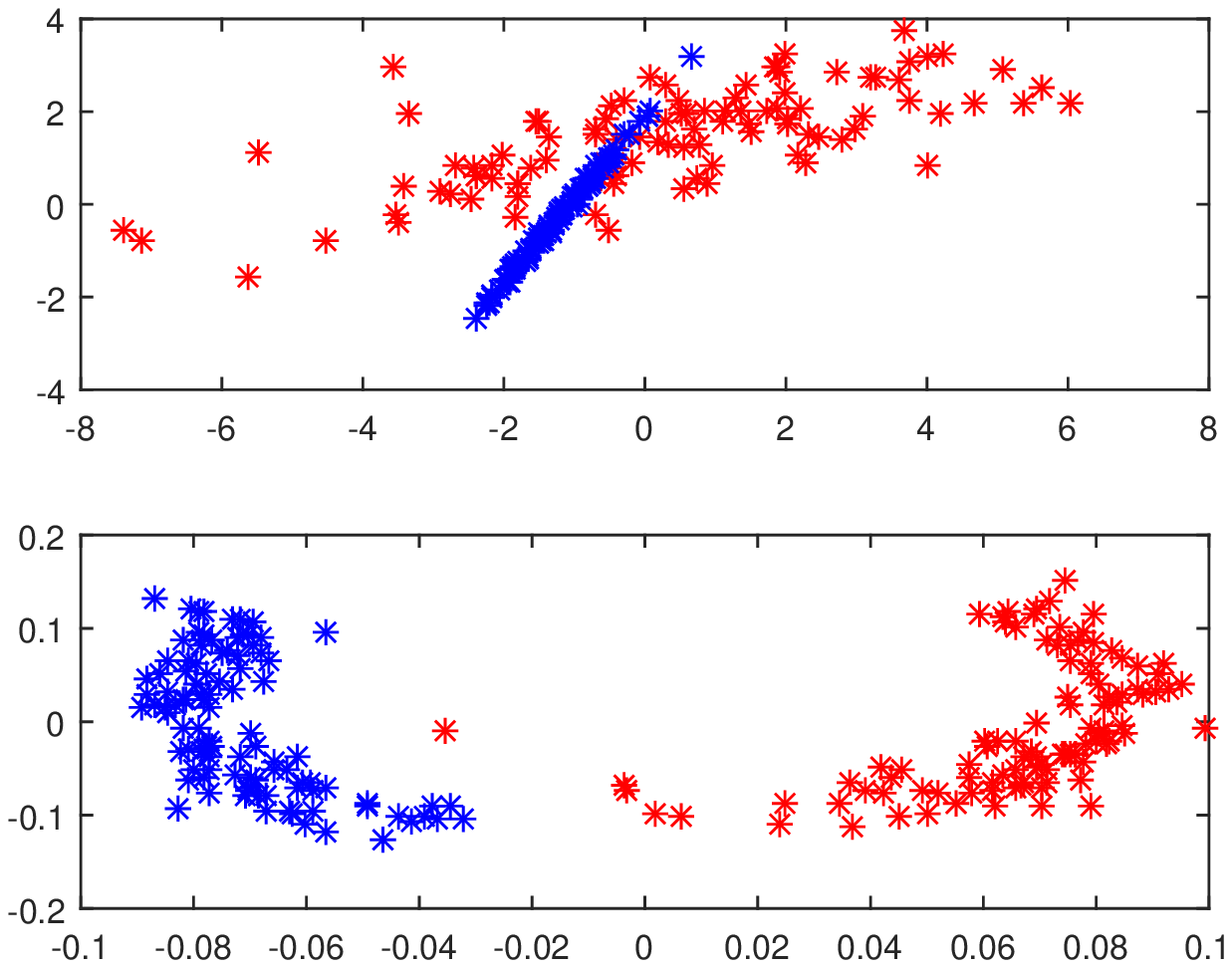}
		\caption{Example 3}
		\label{ex3}
	\end{figure}

	\subsection{Comparison with standard embeddings on a 3D cluster example} 
	
	Simulations have been conducted to assess the quality of the proposed embedding. In this subsection, we used the Matlab package drtoolbox https://lvdmaaten.github.io/drtoolbox/ proposed by 
	Laurens Van Maatten on a sample drawn from a 10 dimensional Gaussian Mixture Model with 4 components and equal proportions. In Figure \ref{GV}, we show the original affinity matrix together with the estimated cluster matrix. In Figure \ref{fig:1}, we compare the affinity matrix of data with the affinity matrix of the mapped data using various embeddings proposed in the drtoolbox package. This toy experiment shows that the embedding described in this paper can cluster as the same time as is embeds into a small dimensional subspace. This is not very surprising since our embedding is taylored for the joint clustering-dimensionality reduction purpose whereas most of the known existing embedding methods aren't. Given the fact that clustered data are ubiquitous in real world data analysis due to the omnipresence of stratified populations, taking the clustering purpose into account might be a considerable advantage.    
	
	\begin{figure}[!t]
		\centering
		\includegraphics[scale=.5]{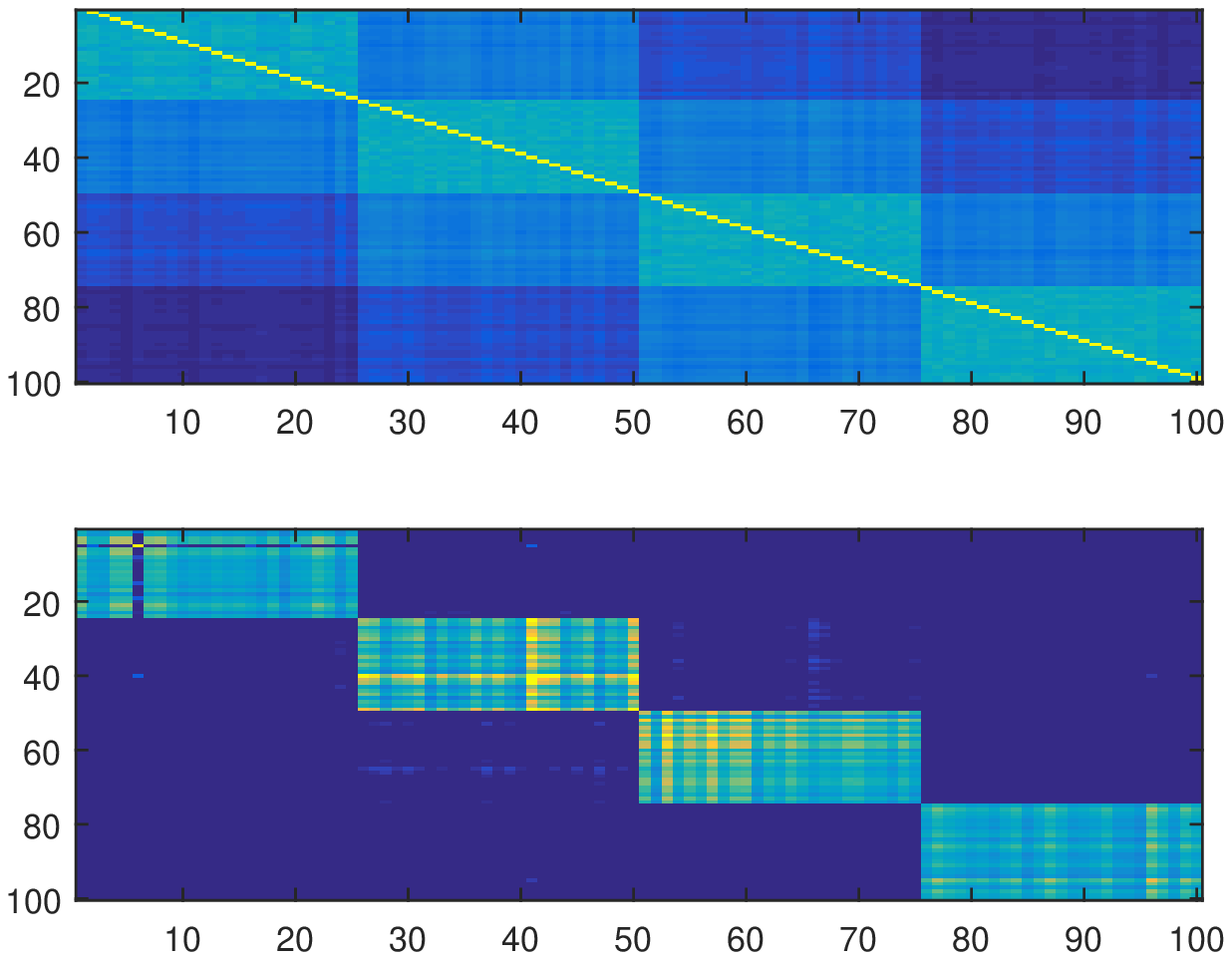}
		\caption{Original affinity matrix vs. Guedon Vershynin Cluster matrix}
		\label{GV}
	\end{figure}

	\begin{figure}[htb]
		\centering
		\subfloat[Original affinity matrix vs. affinity matrix after PCA embedding]{%
			\includegraphics[width=.24\textwidth]{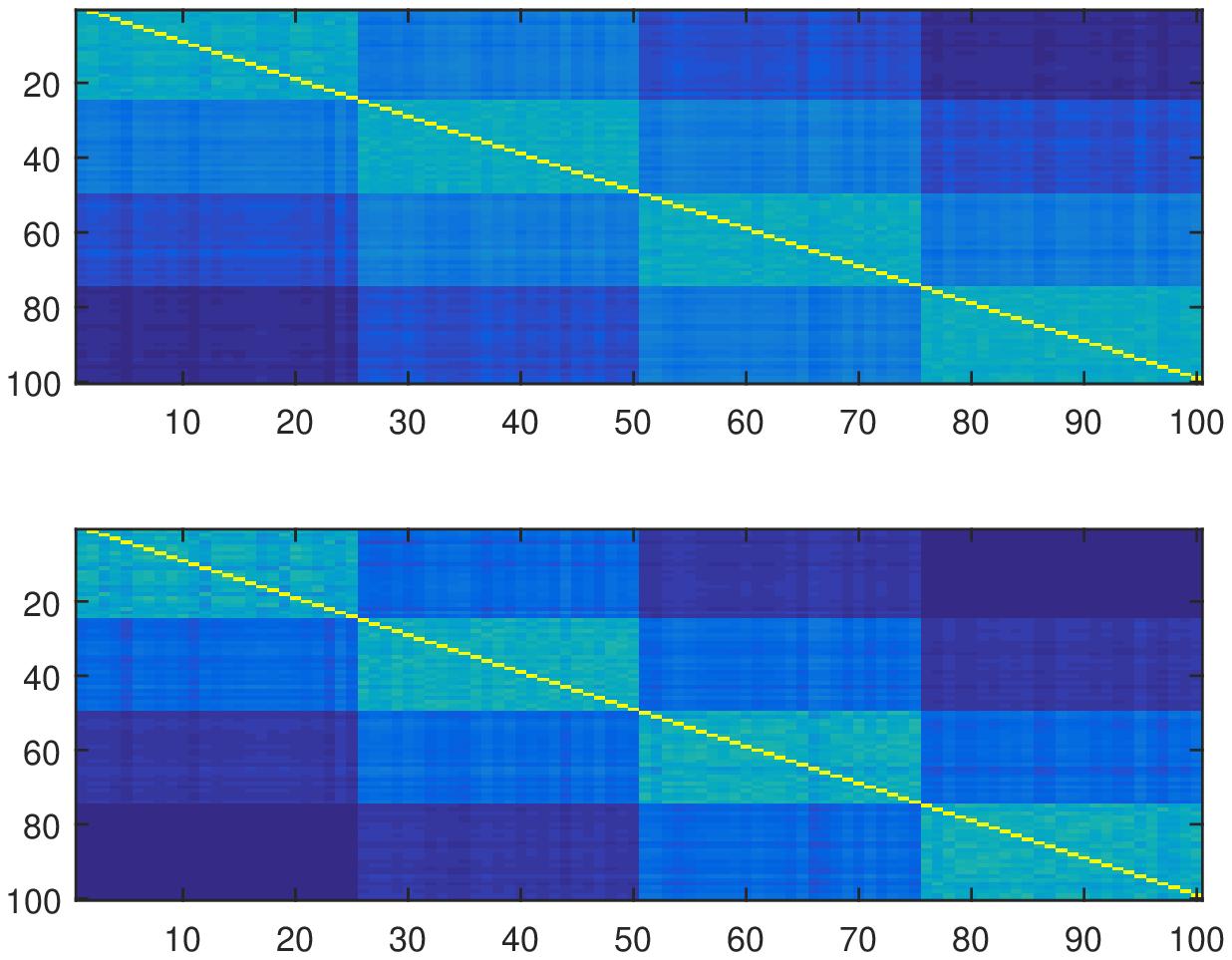}}\hfill
		\subfloat[Original affinity matrix vs. affinity matrix after MDS embedding]{%
			\includegraphics[width=.24\textwidth]{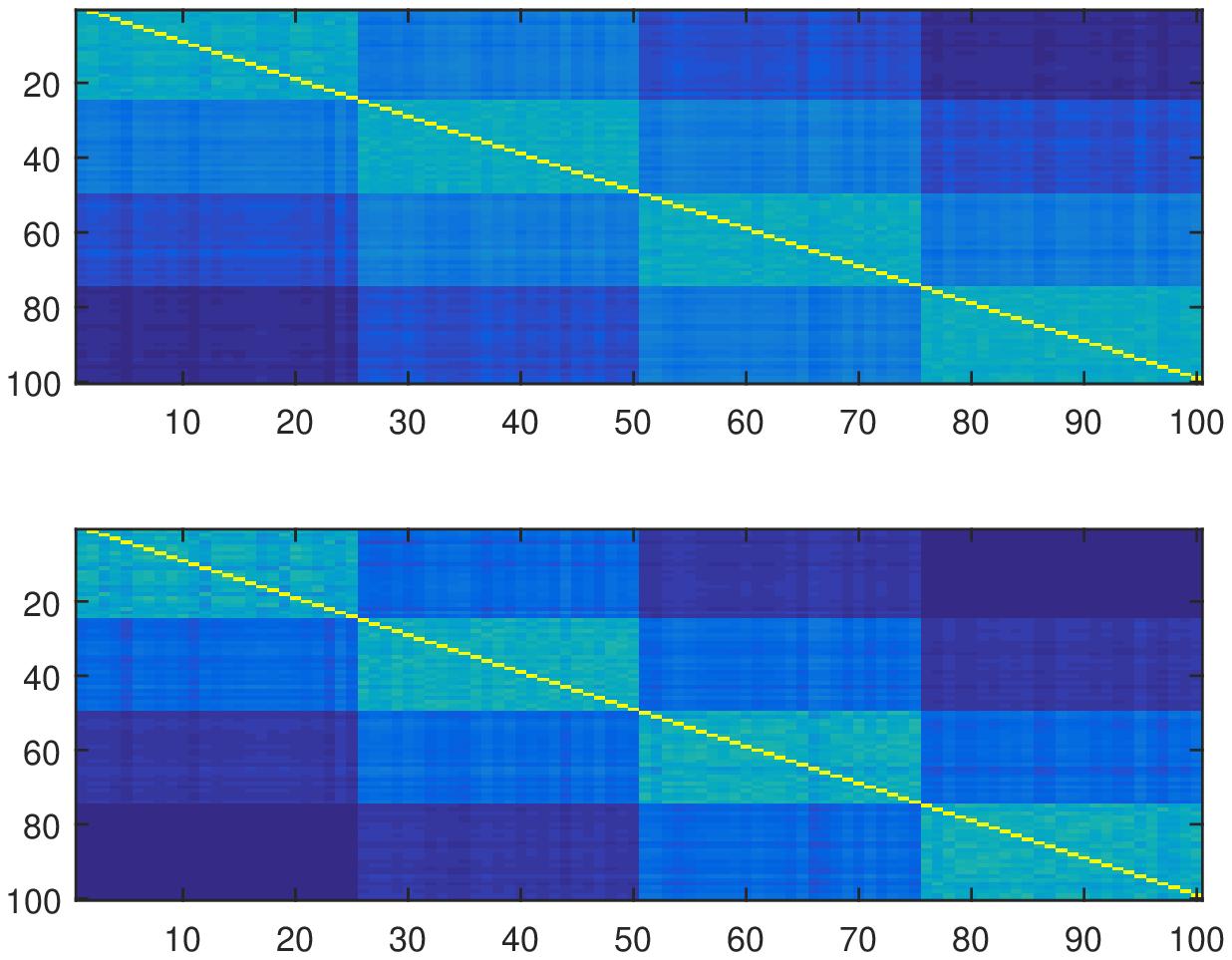}}\hfill
		\subfloat[Original affinity matrix vs. affinity matrix after Factor Analysis embedding]{%
			\includegraphics[width=.24\textwidth]{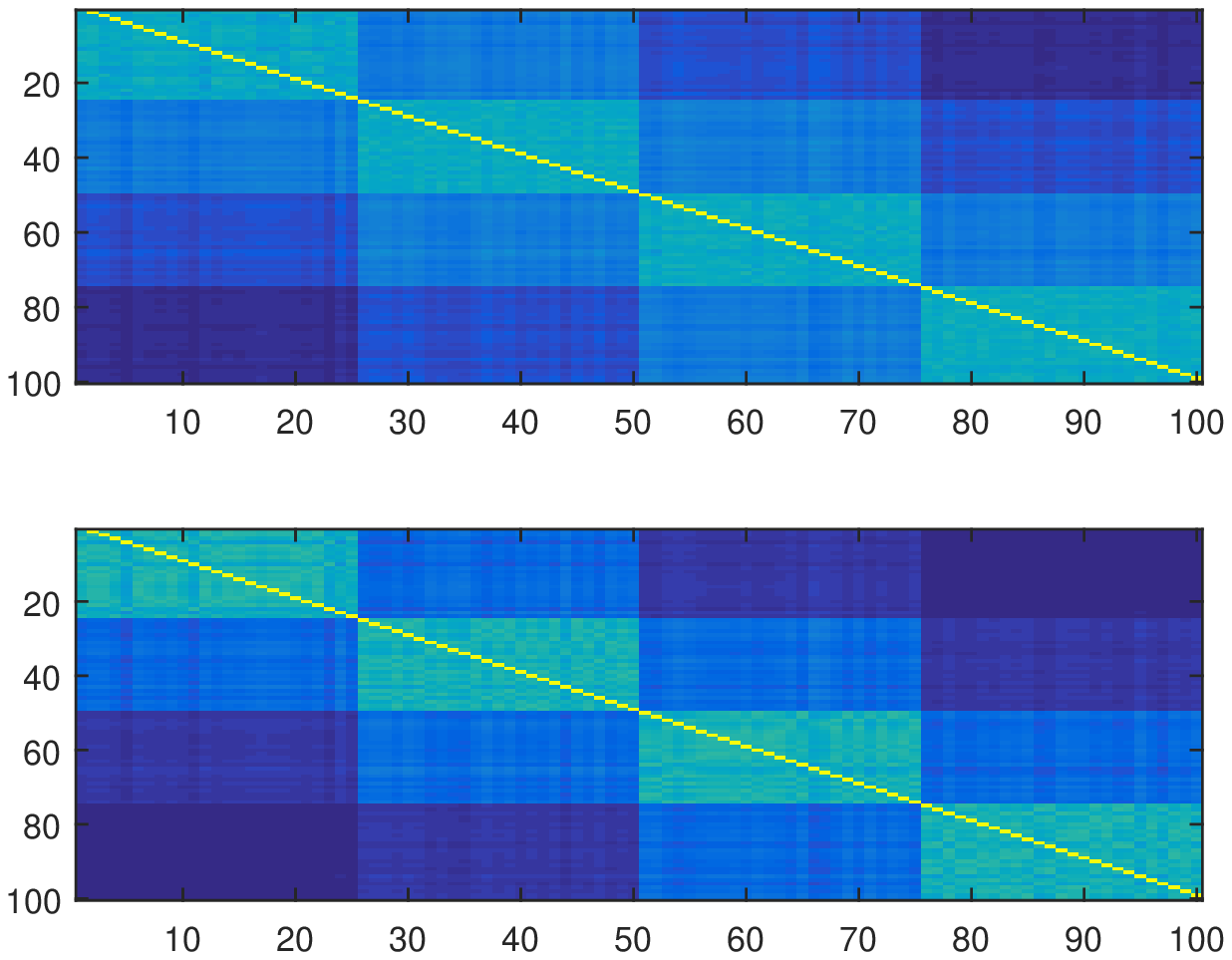}}\hfill
		\subfloat[Original affinity matrix vs. affinity matrix after t-SNE
		embedding]{%
			\includegraphics[width=.24\textwidth]{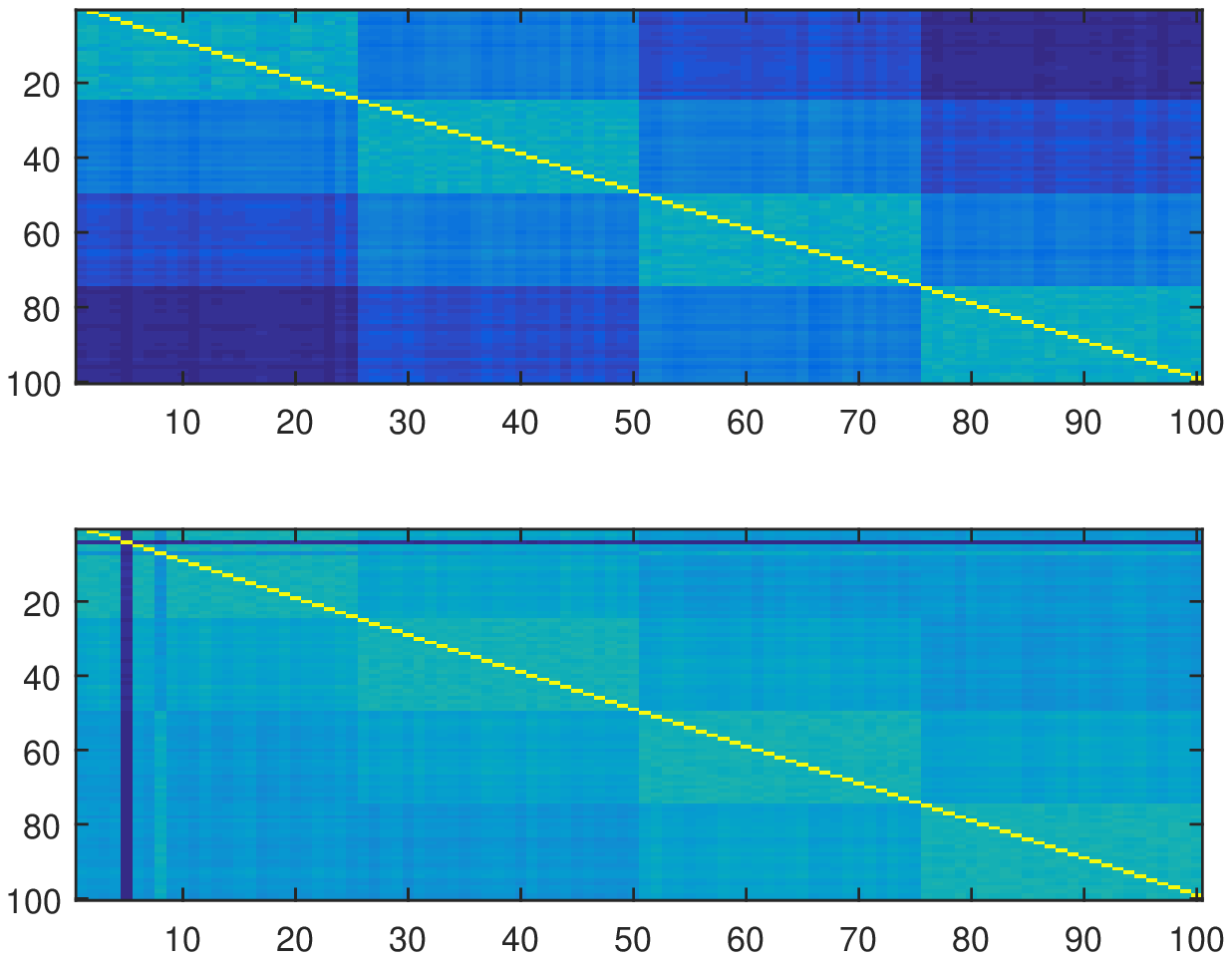}} \\
		\subfloat[Original affinity matrix vs. affinity matrix after Sammon
		embedding]{%
			\includegraphics[width=.24\textwidth]{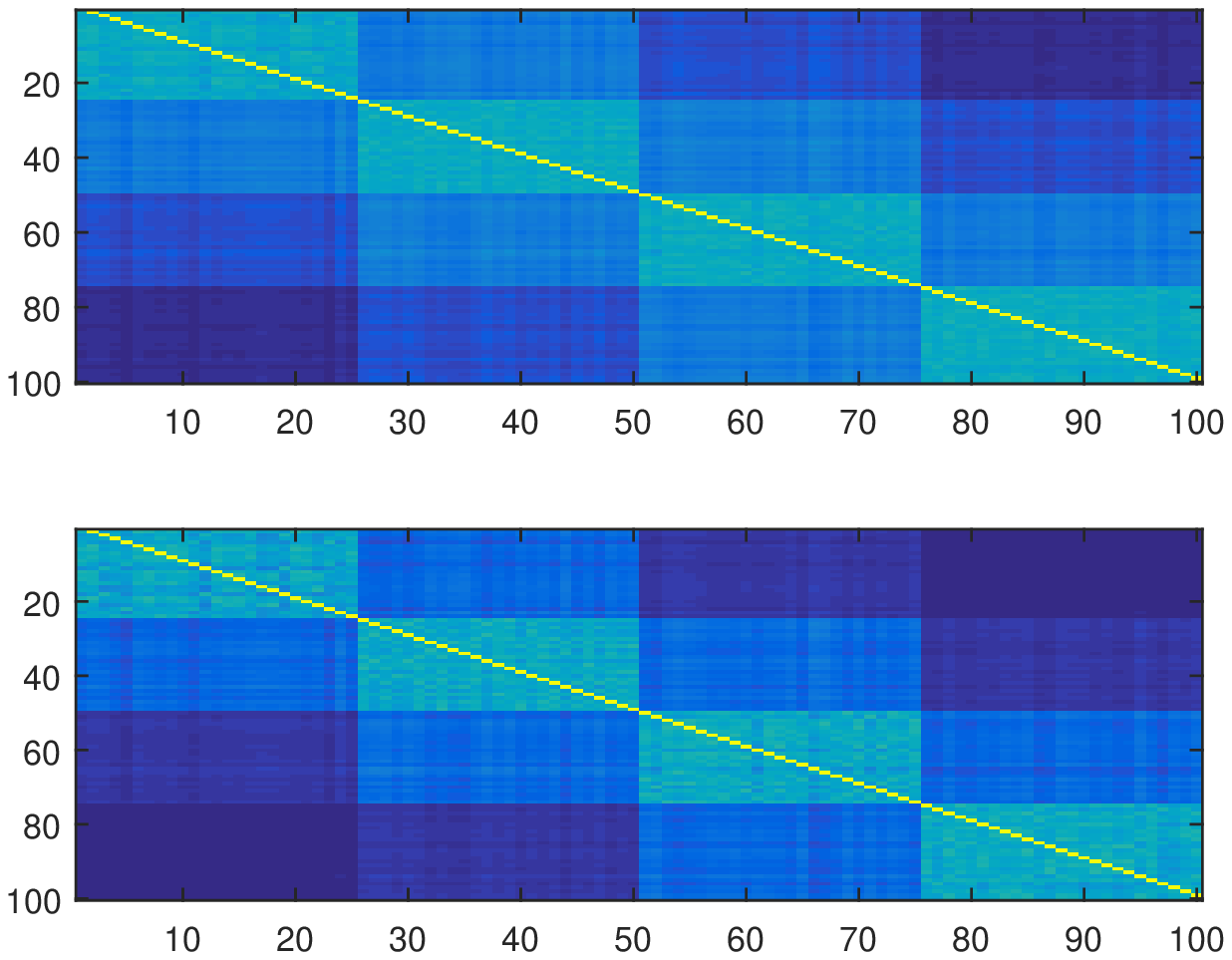}}\hfill
		\subfloat[Original affinity matrix vs. affinity matrix after LLE
		embedding]{%
			\includegraphics[width=.24\textwidth]{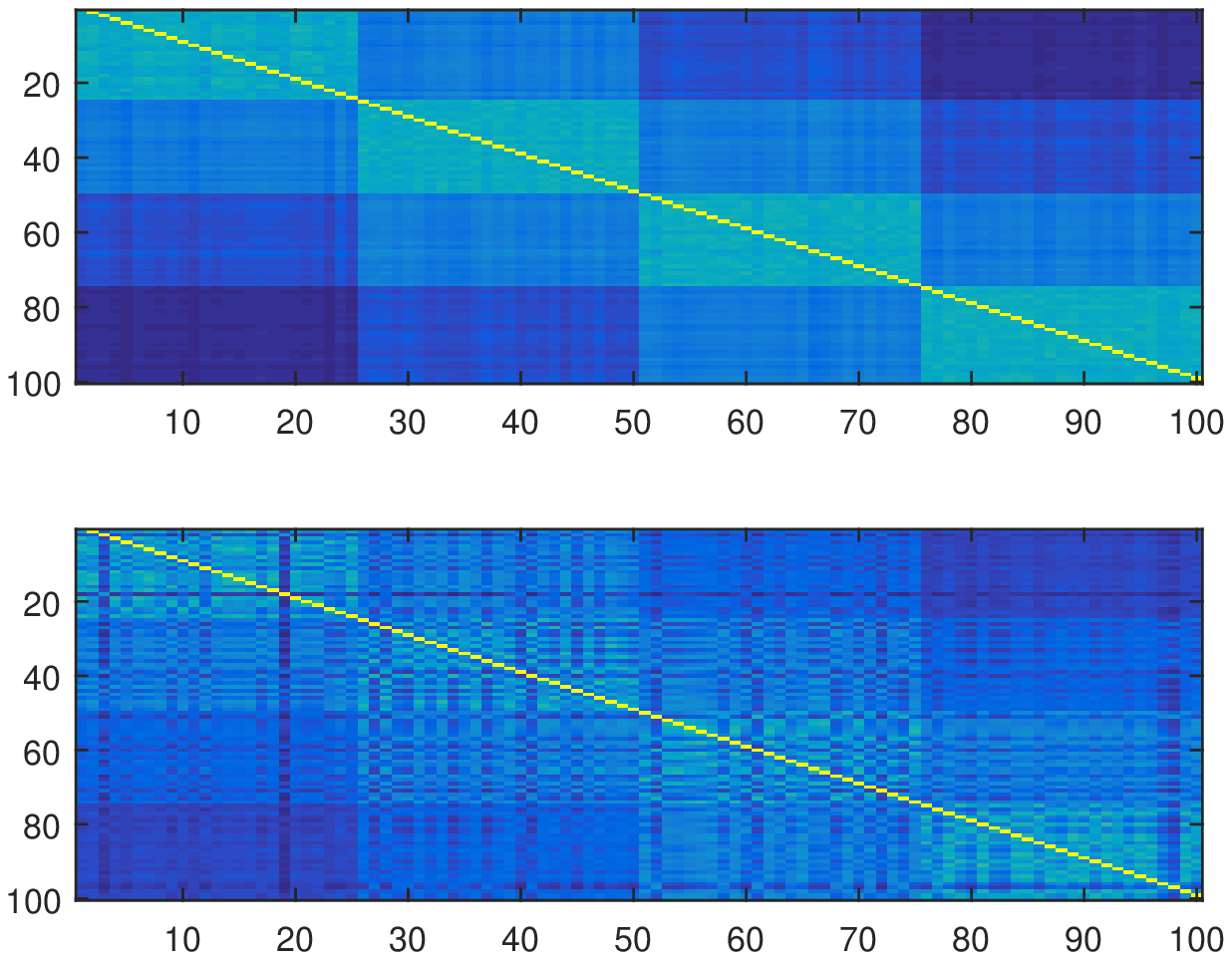}}\hfill
		\subfloat[Original affinity matrix vs. affinity matrix after Laplacian embedding]{%
			\includegraphics[width=.24\textwidth]{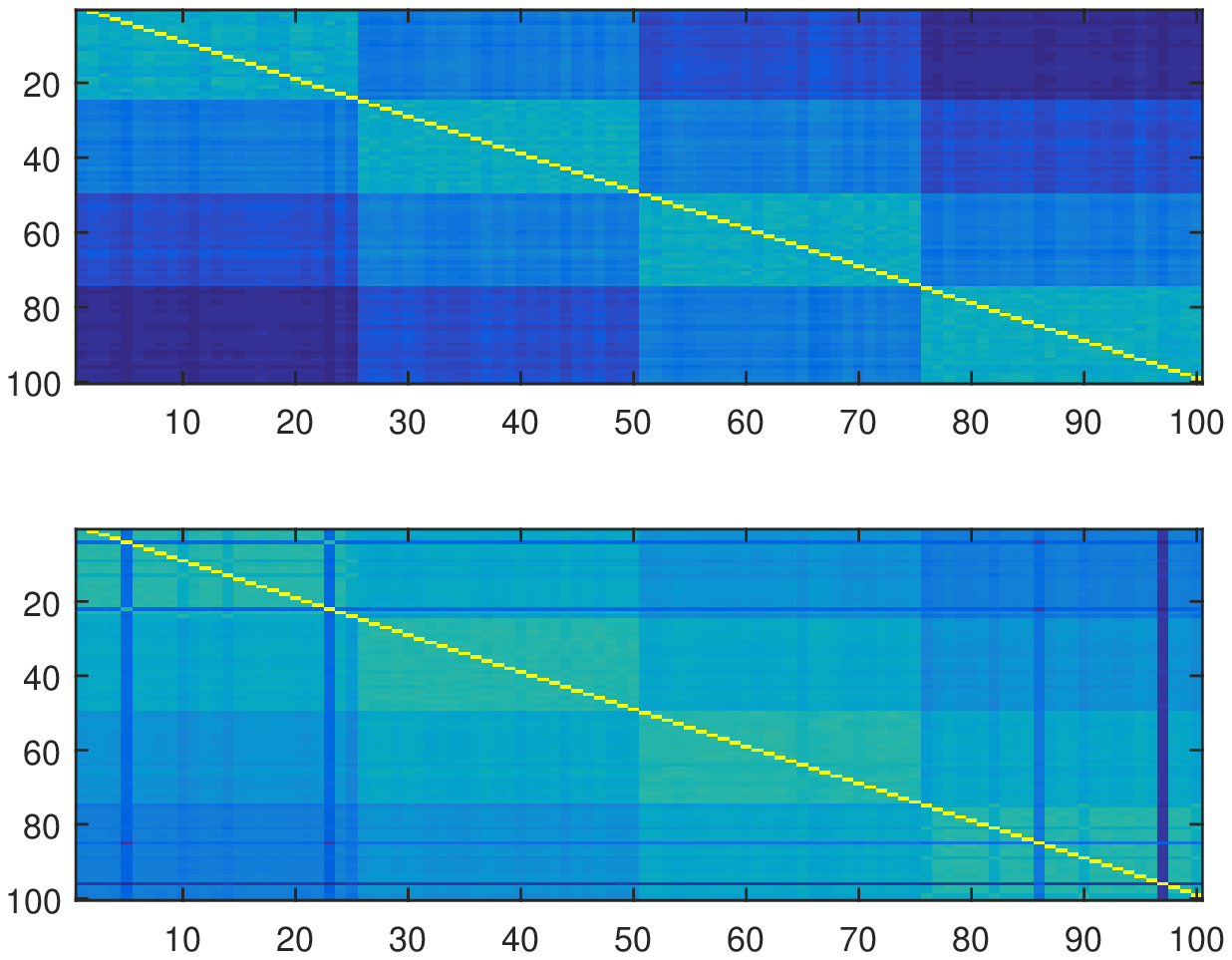}}\hfill
		\subfloat[Original affinity matrix vs. affinity matrix after Kernel-PCA embedding]{%
			\includegraphics[width=.24\textwidth]{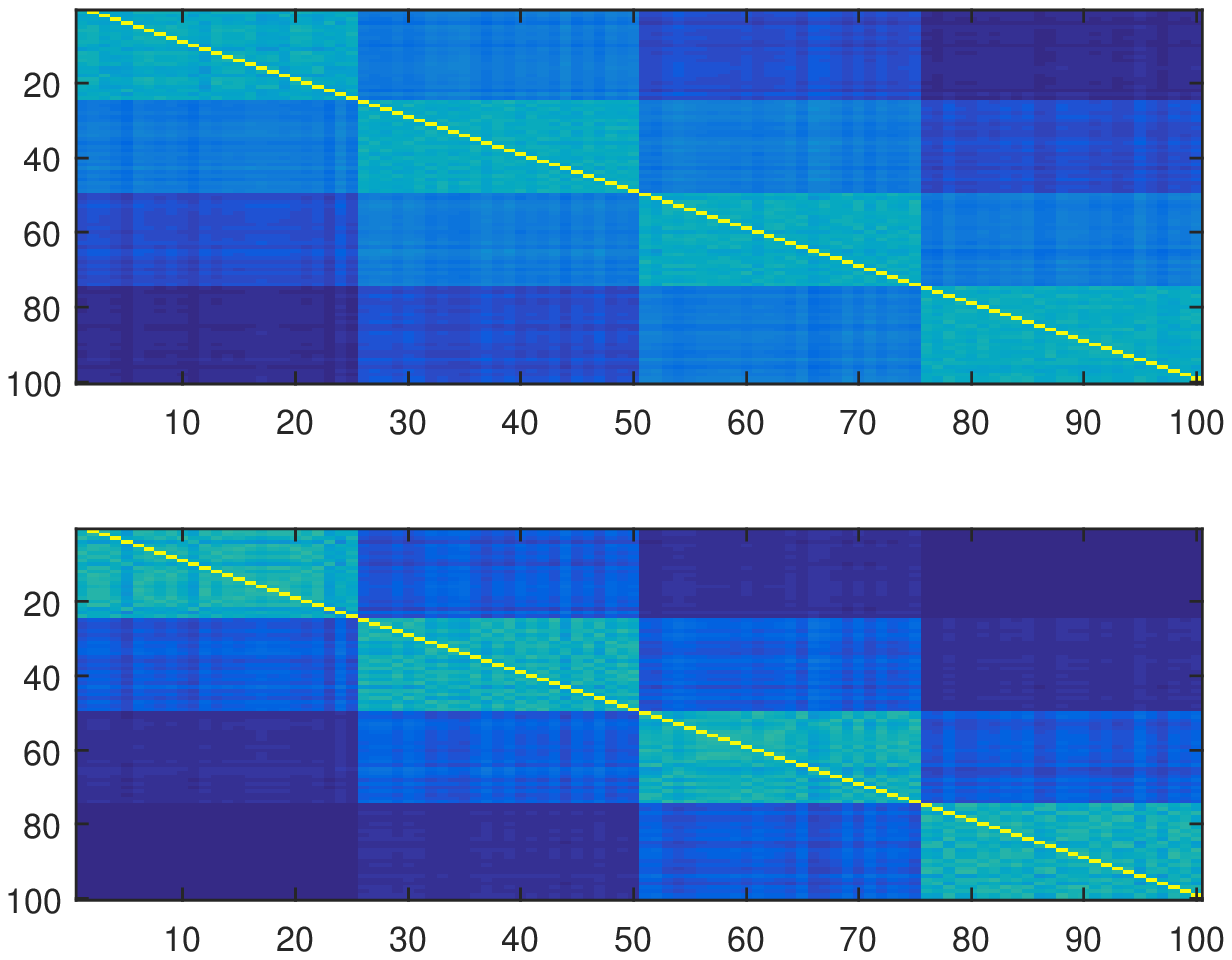}}\\
		\subfloat[Original affinity matrix vs. affinity matrix after LTSA
		embedding]{%
			\includegraphics[width=.24\textwidth]{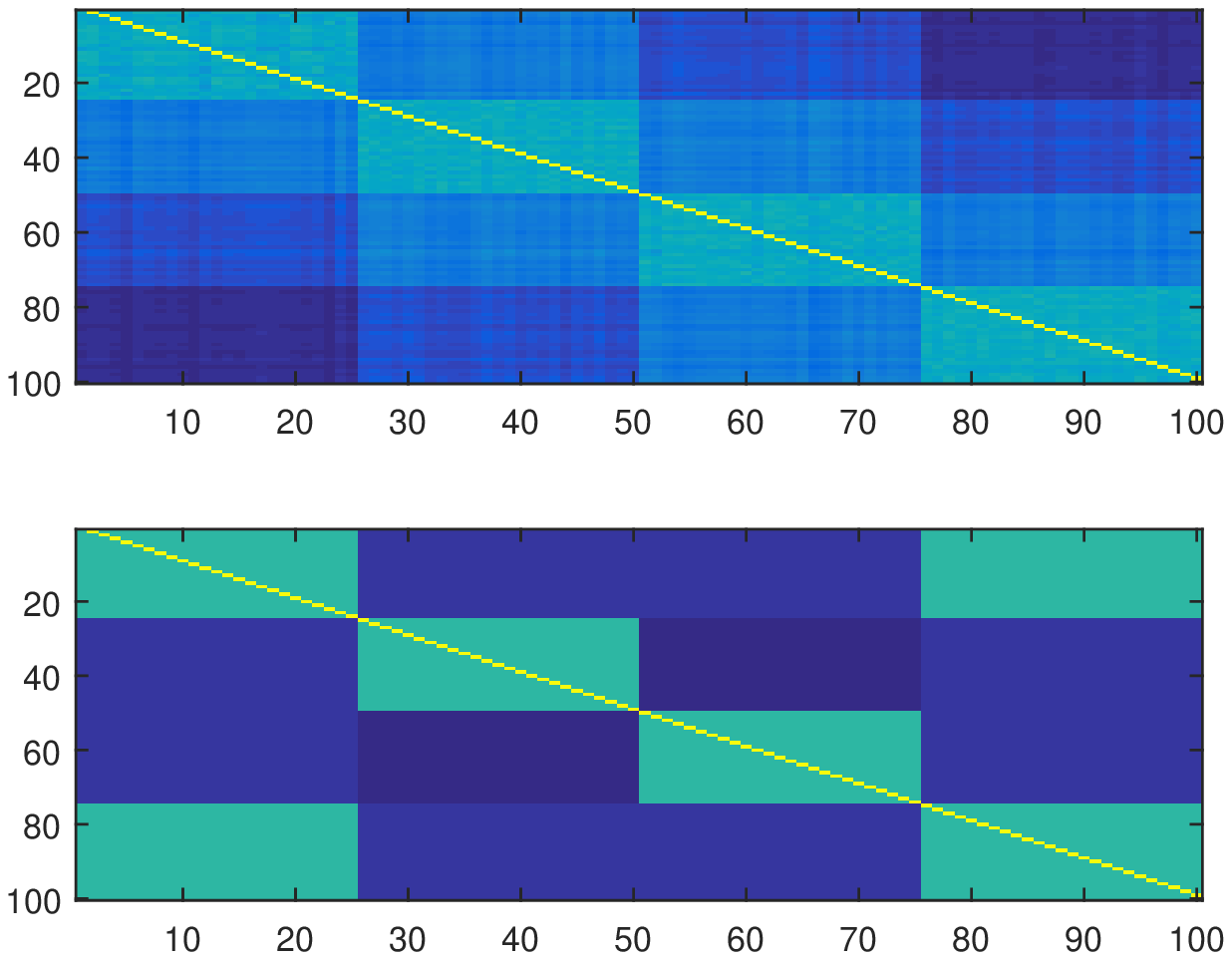}}\hfill
		\subfloat[Original affinity matrix vs. affinity matrix after MVU
		embedding]{%
			\includegraphics[width=.24\textwidth]{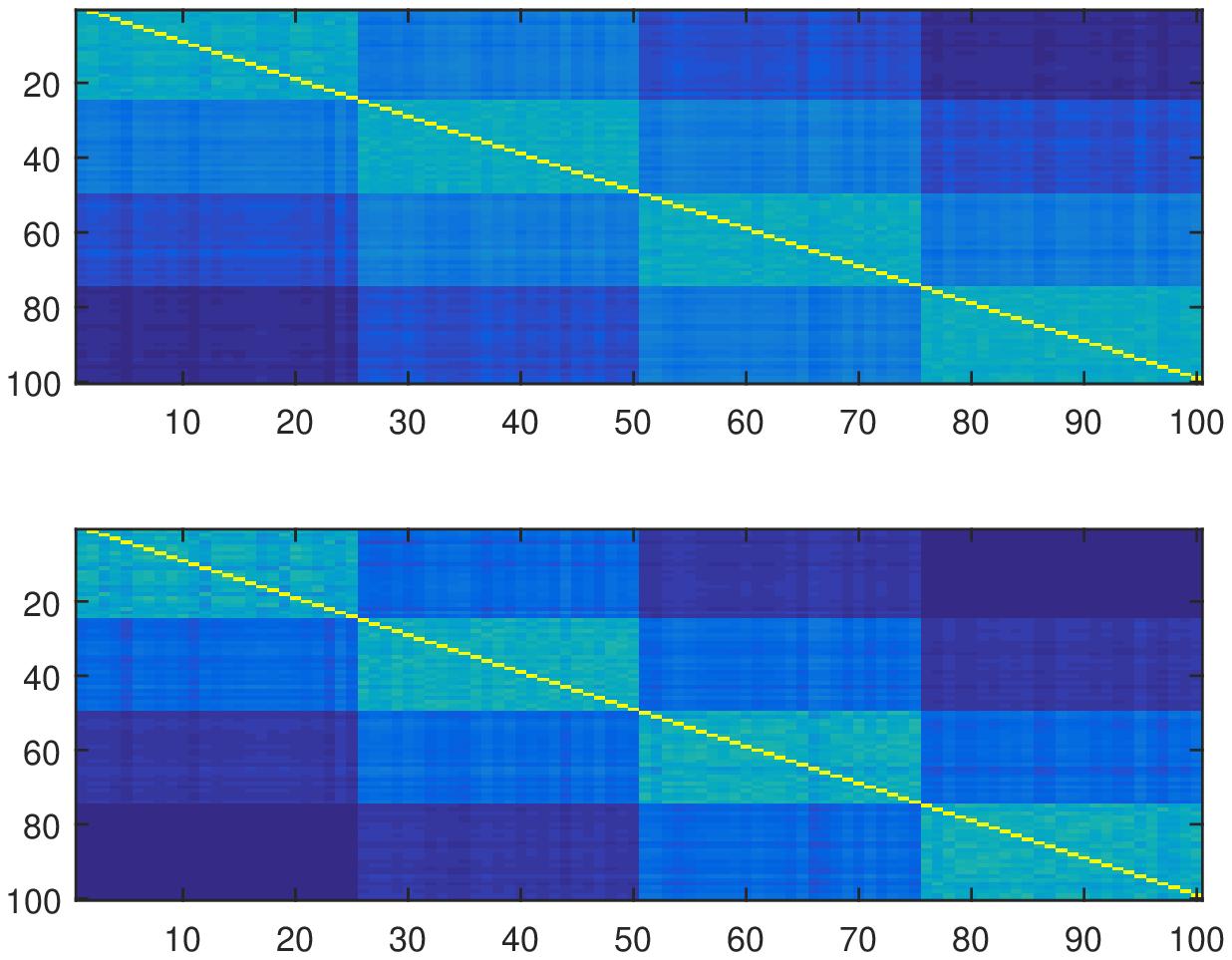}}\hfill
		\subfloat[Original affinity matrix vs. affinity matrix after Auto Encoder embedding]{%
			\includegraphics[width=.24\textwidth]{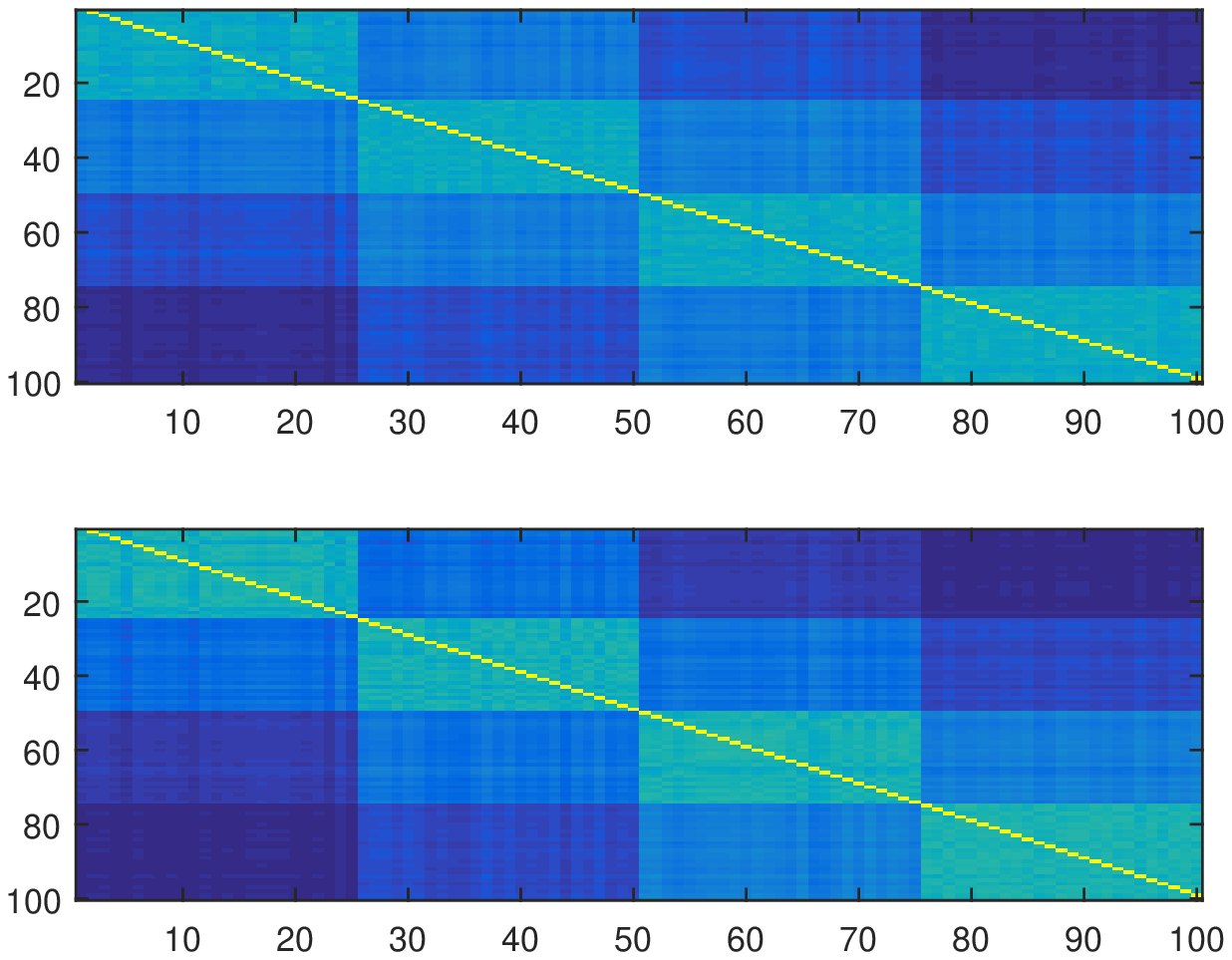}}\hfill
		\subfloat[Original affinity matrix vs. affinity matrix after DiffusionMap
		embedding]{%
			\includegraphics[width=.24\textwidth]{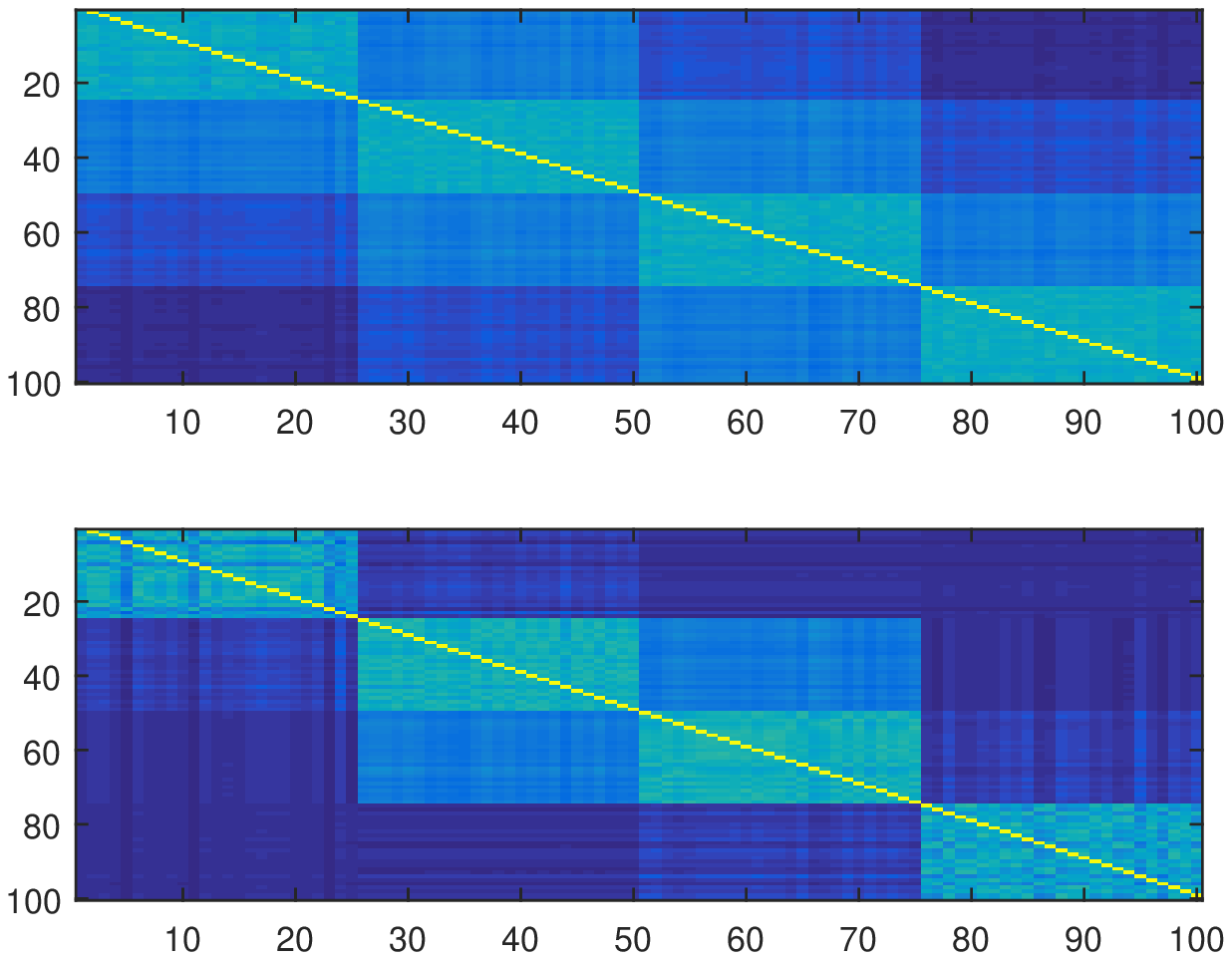}}\\
		
		\caption{The affinity matrix obtained after embedding using different methods from the Matlab package drtoolbox https://lvdmaaten.github.io/drtoolbox/}\label{fig:1}
	\end{figure}

	\subsection{The sparsity of the solution}
	
	In this experiment, we computed the relative sparsity of the affinity matrix of the embedded data v.s. the sparsity of the affinity matrix of the original data.
	
	Since the matrix is not exactly sparse, we chose to study the $l^p$ quasi-norm instead of the exact sparsity
	for $p$ small, i.e. $p=.05$. 
	
	We made 250 Monte Carlo experiments in dimension 10, 30, 50, 70 and 90. The histogram of the relative error of the $l^p$ quasi-norm for each dimension is given in Figure \ref{histolp}. In this experiment, we draw the sample from two Gaussian distributions with mean drawn from 
	$\mathcal N(0,2I)$ and covariance drawn as $A^A$ where the components of $A$ are i.i.d. 
	$\mathcal N(0,1)$. The affinity matrix of the orignal data is already quite sparse but the embedding improves the sparsity by 10 to nearly 20 percent as the dimension increases.

	\begin{figure}[htb]
		\centering
		\subfloat[$p=.05$ dimens. =50]{%
			\includegraphics[width=.30\textwidth]{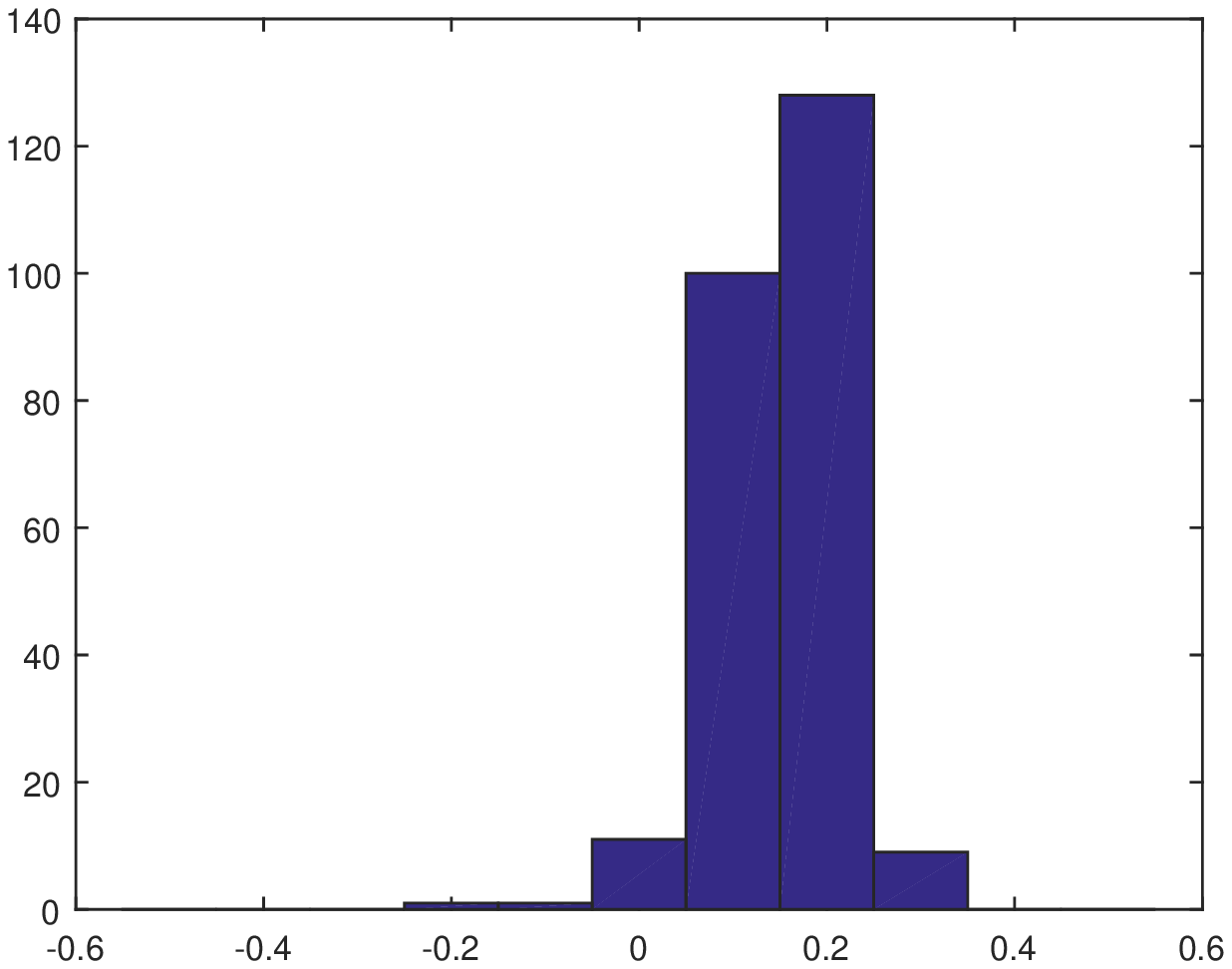}}\hfill
		\subfloat[$p=.05$  dimens. =100]{%
			\includegraphics[width=.30\textwidth]{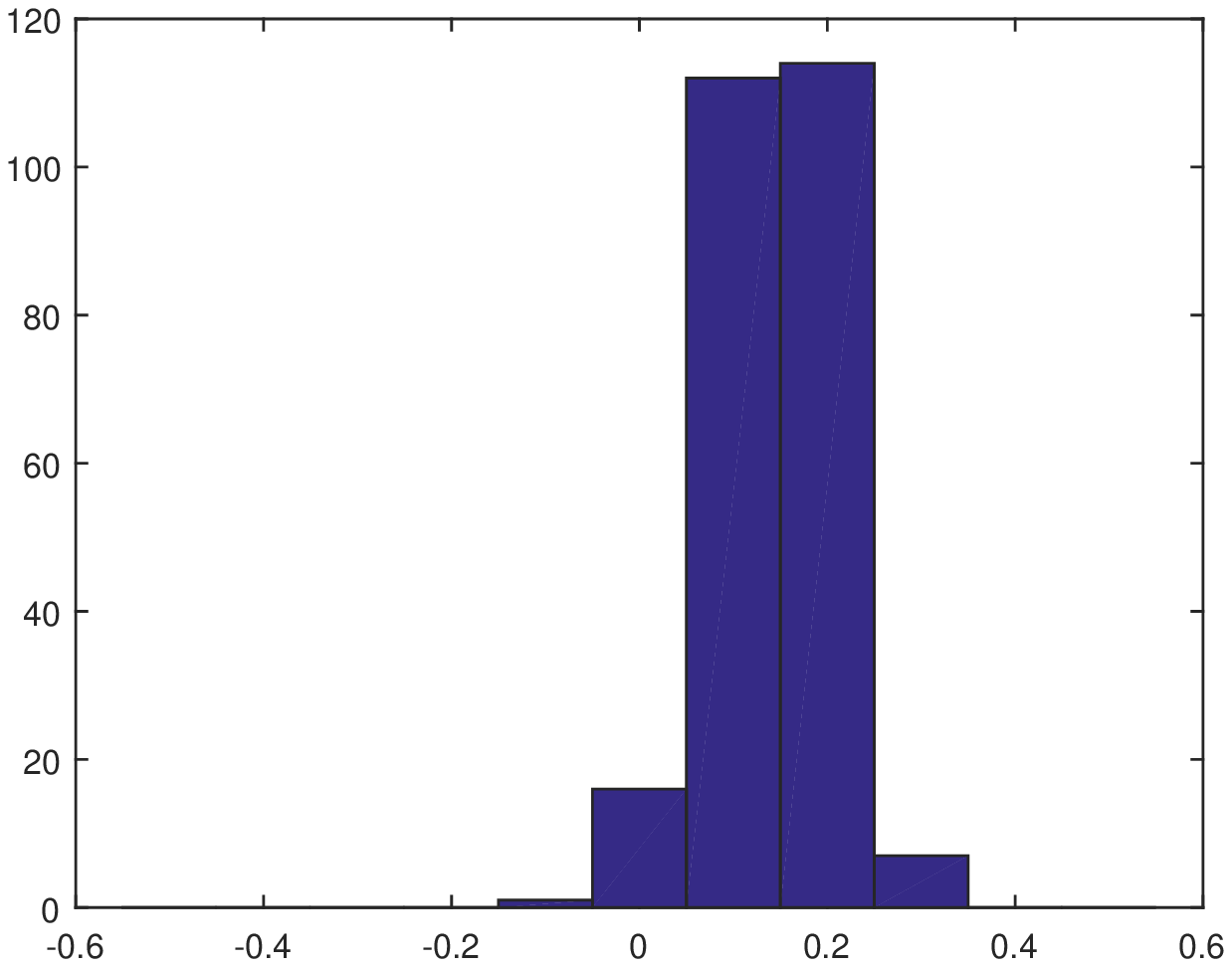}}\hfill
		\subfloat[$p=.05$ in dimens. 150]{%
			\includegraphics[width=.30\textwidth]{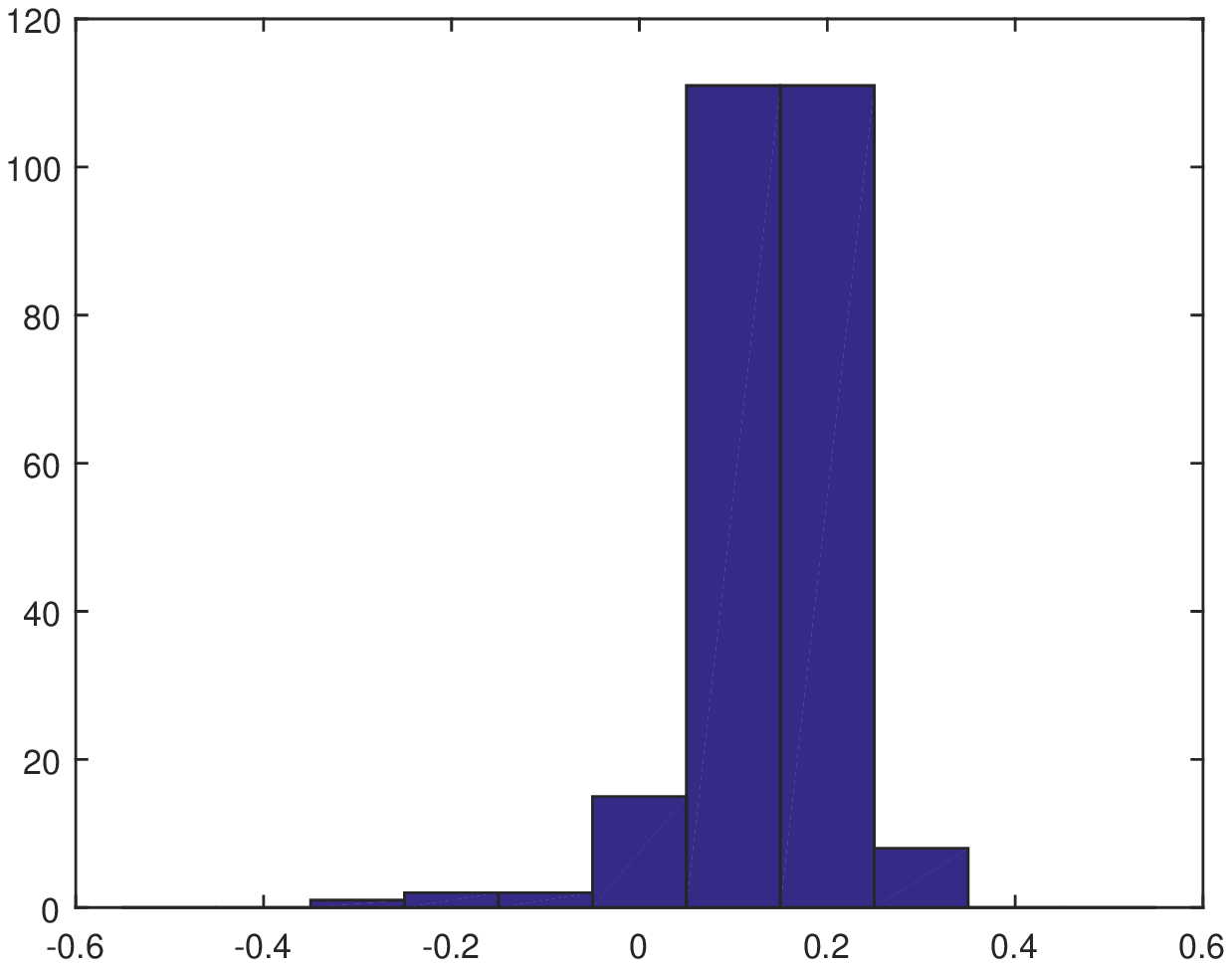}}\\
		\subfloat[$p=.05$ dimens. =200]{%
			\includegraphics[width=.30\textwidth]{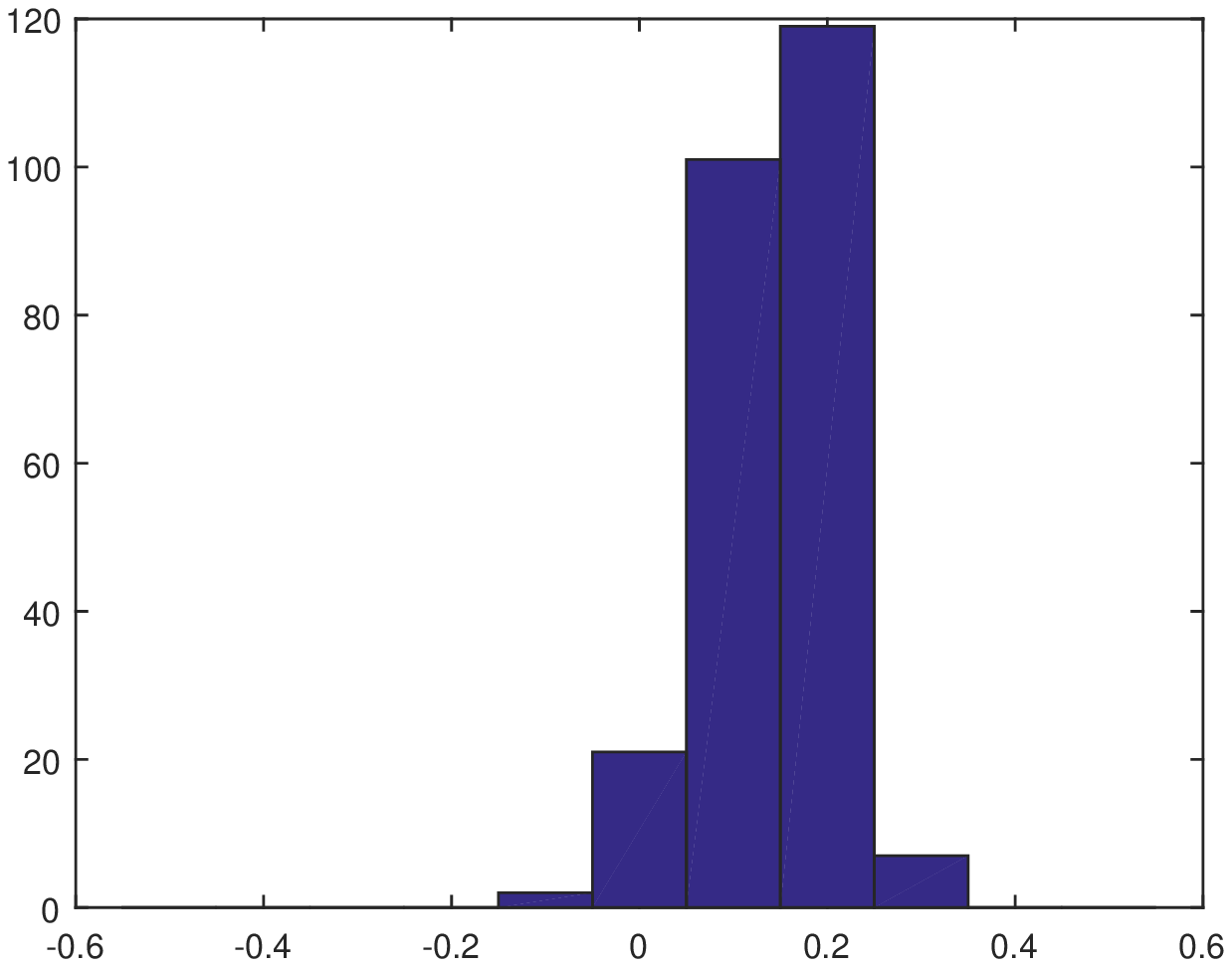}} \hfill
		\subfloat[$p=.05$ dimens.= 250]{%
			\includegraphics[width=.30\textwidth]{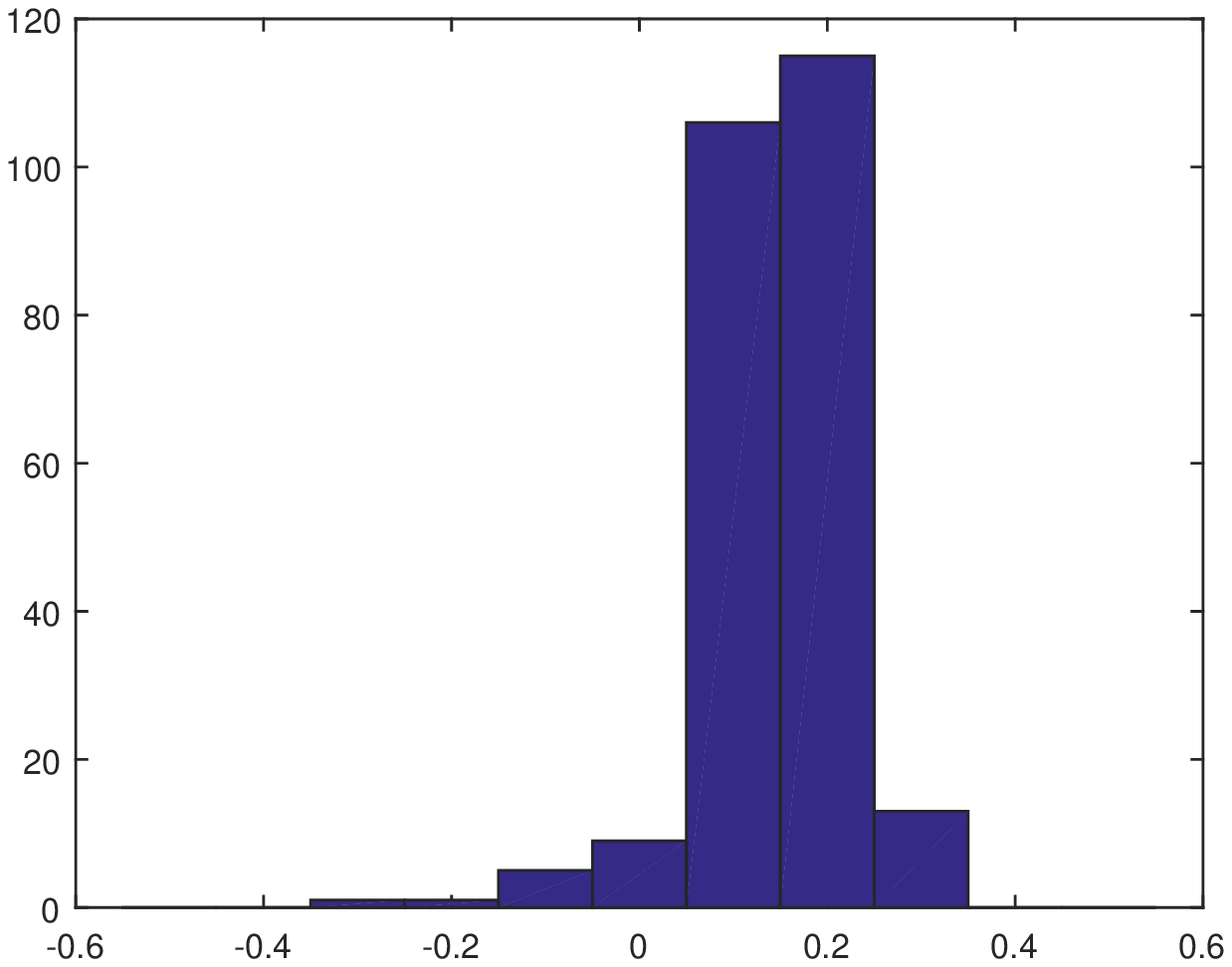}}\hfill
		\subfloat[$p=.05$ dimens. =300]{%
			\includegraphics[width=.30\textwidth]{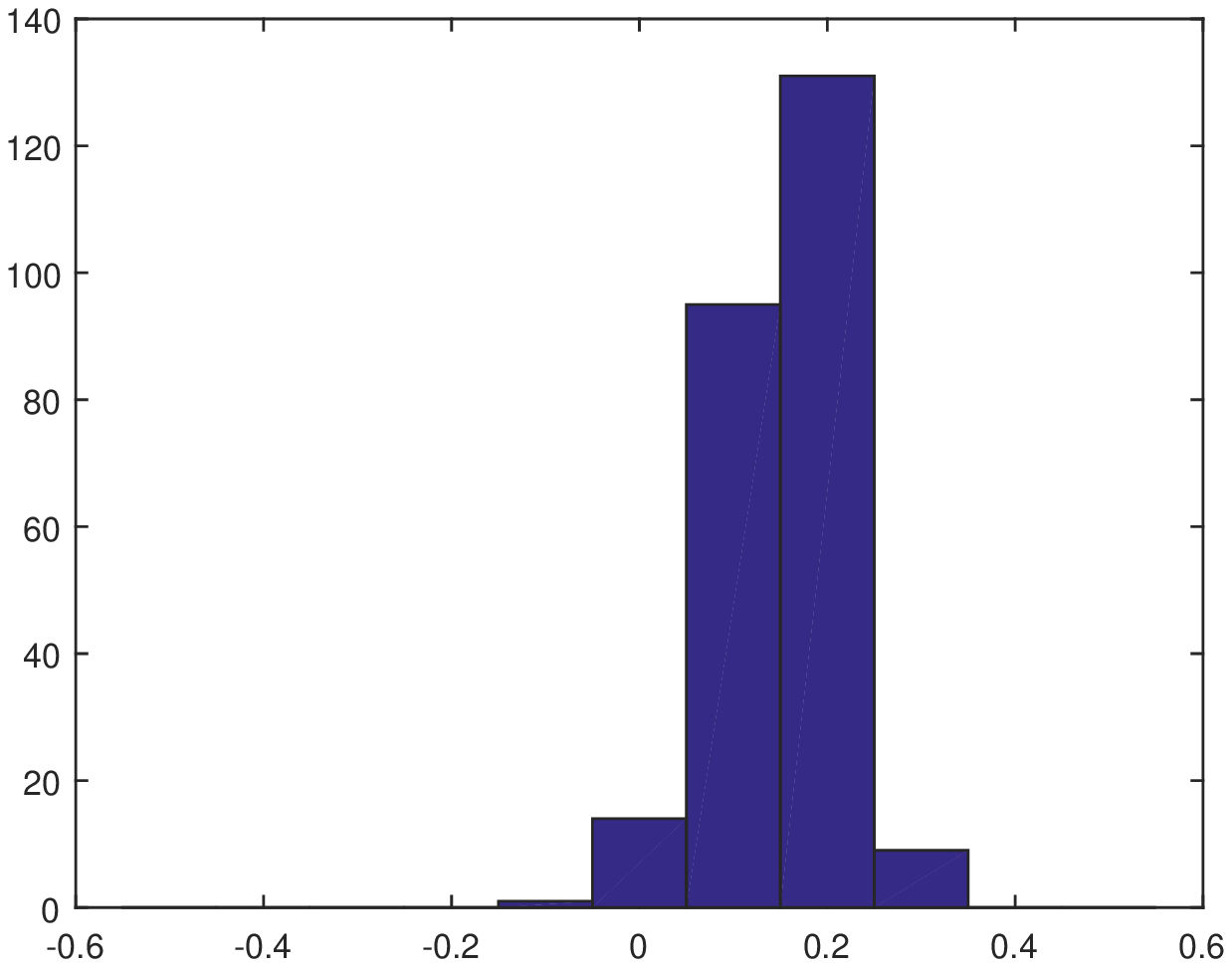}}\\
		\subfloat[$p=.05$ dimens. =350]{%
			\includegraphics[width=.30\textwidth]{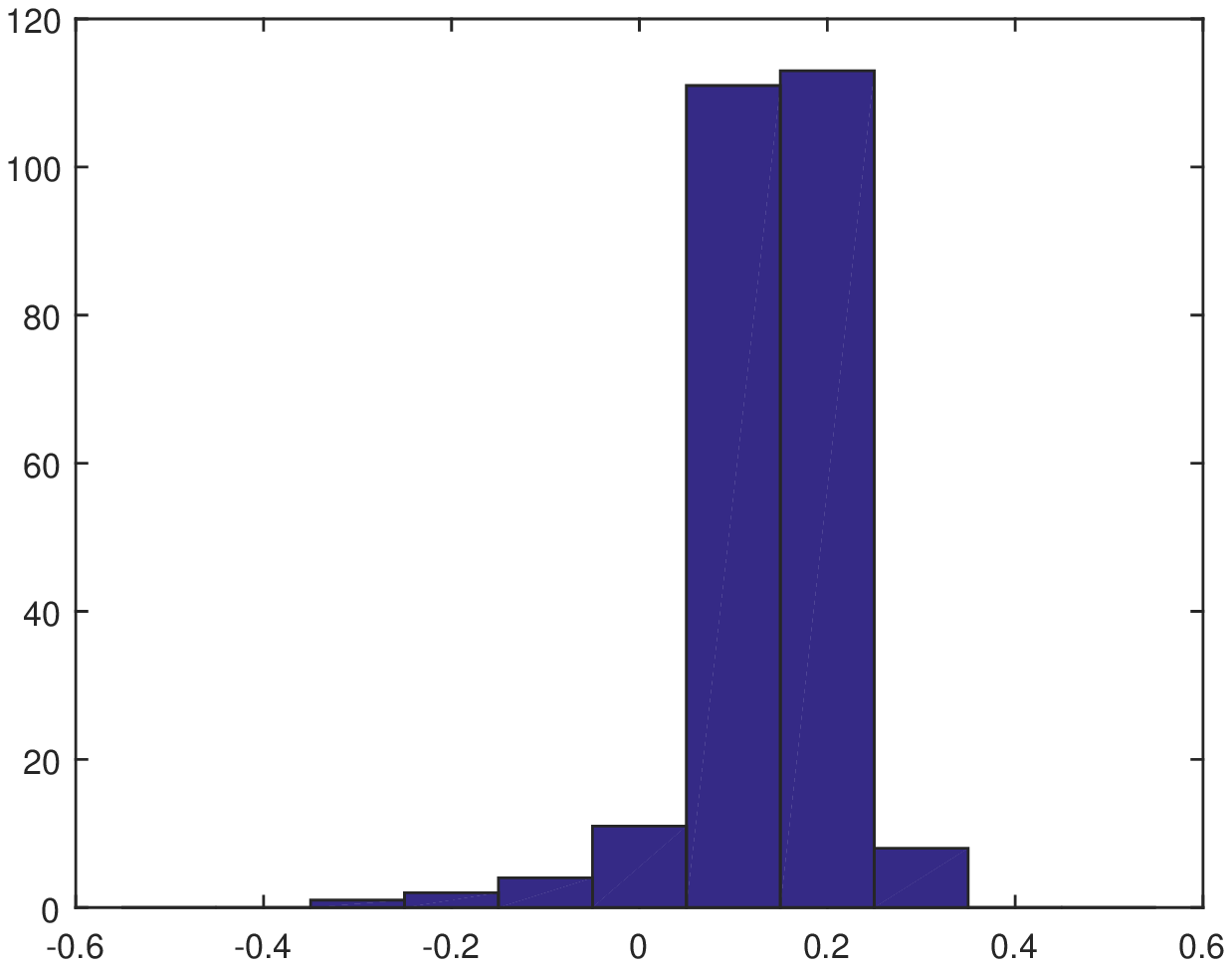}}\hfill
		\subfloat[$p=.05$ dimens. = 400]{%
			\includegraphics[width=.30\textwidth]{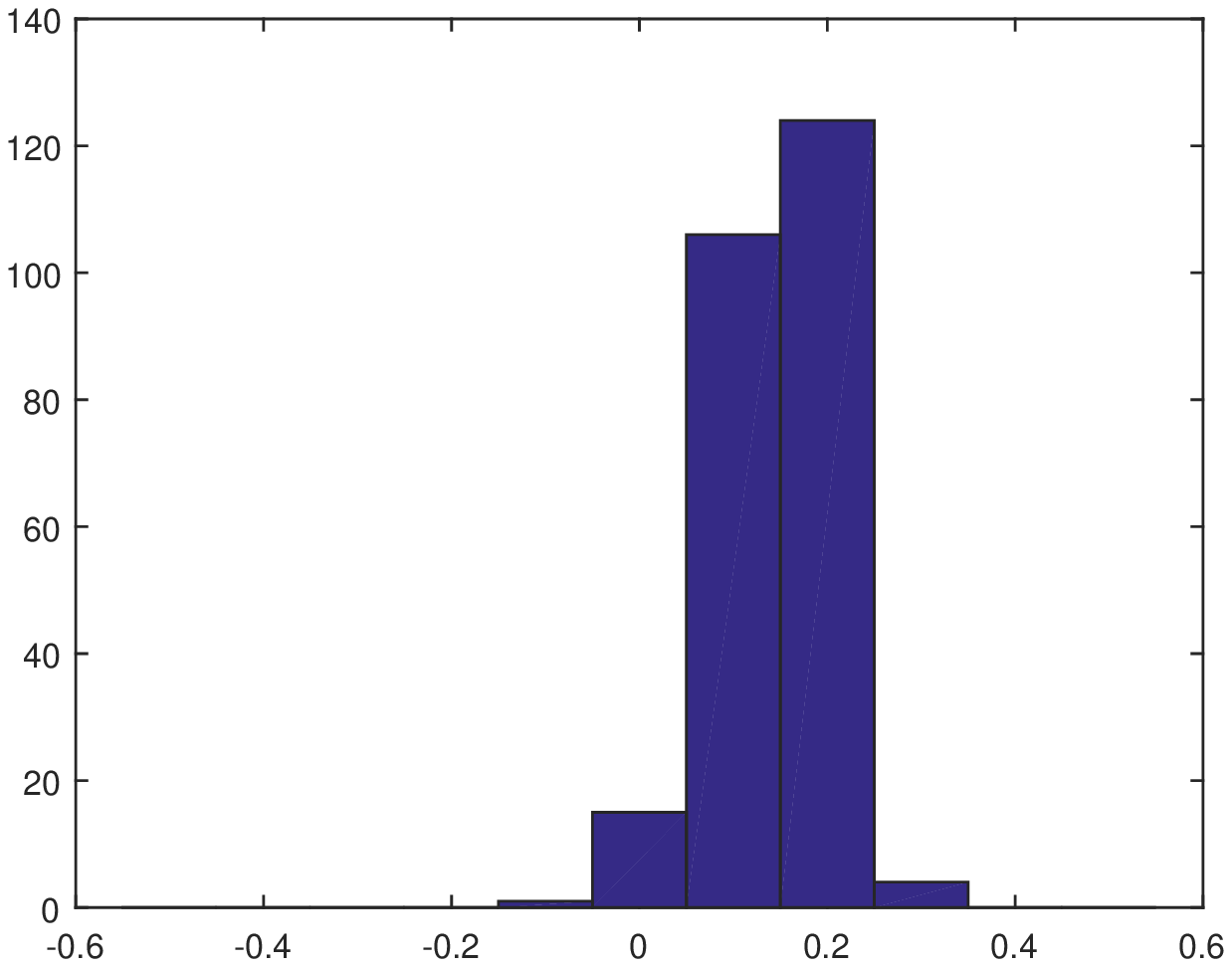}}\hfill
		\subfloat[$p=.05$ dimens. =450]{%
			\includegraphics[width=.30\textwidth]{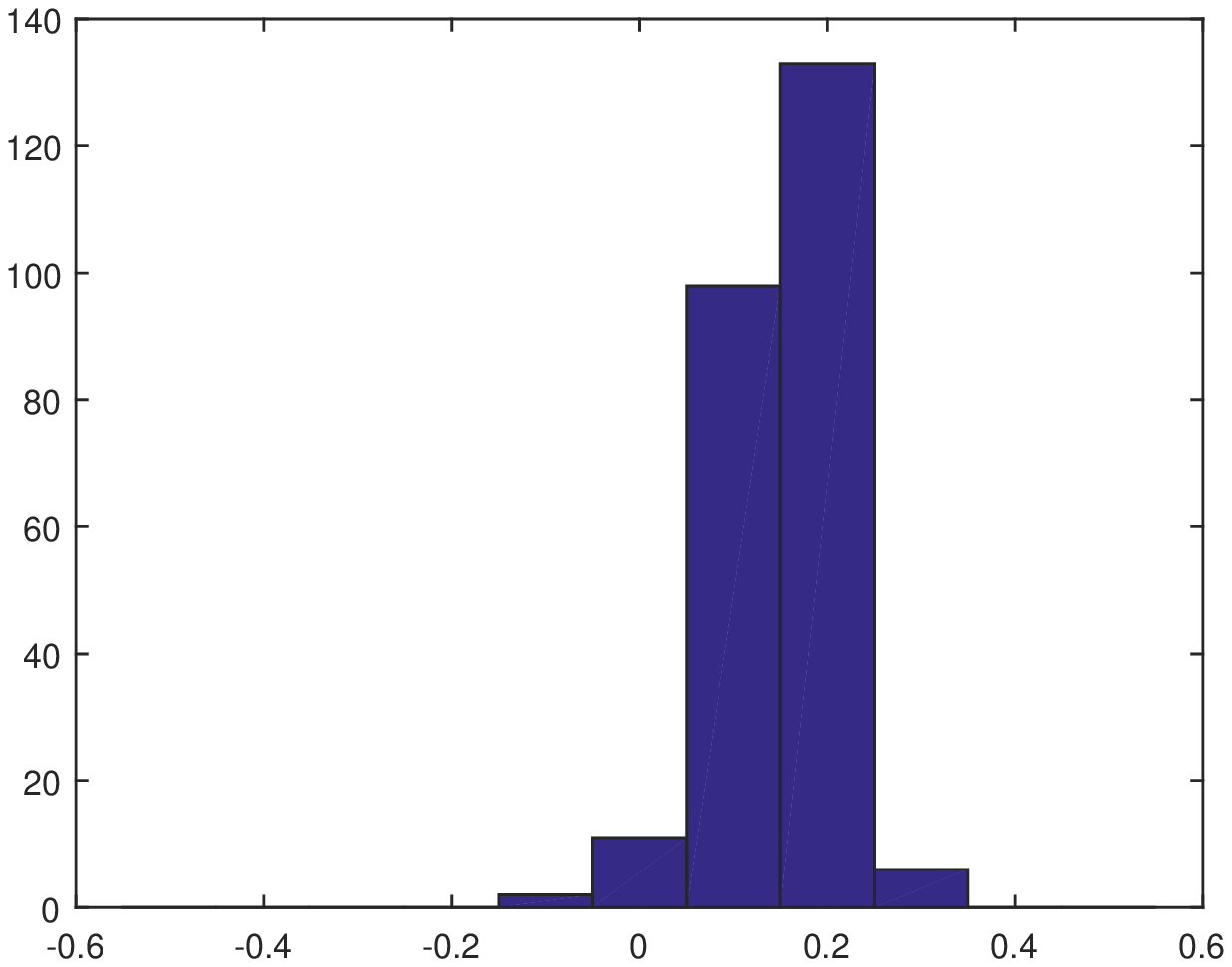}}
		
		\caption{Histogram of the relative difference of the $l^p$ norm of the affinity matrix between the original data and the embedded data}\label{histolp}
	\end{figure}

	\section{Conclusions}
	
	The goal of the present paper was to propose an analysis of Guedon and Vershynin's Semi-Definite Programming approach to the estimation of the 
	cluster matrix and show how this matrix can be used to produce an embedding 
	for preconditionning standard clustering procedures. The procedure is suitable for very high dimensional data because it is based on pairwise distances
	only. Moreover, increasing the dimension will improve the robustness of the procedure as soon as a Law of Large Numbers holds along the variables instead of the samples, forcing the affinity matrix to converge to a deterministic limit and thus making the estimator less sensitive to its low dimensional fluctuations. 
	
	Another feature of the method is that it may apply to a large number of mixtures type, even when the component's densities are not log-concave, as do a lot of embeddings as applied to data concentrated on complicated manifolds. 
	Further studies will be performed in this exciting direction. 
	
	Future work is also needed for proving that the proposed preconditioner is provably efficient when combined with various clustering techniques. One of the 
	main reason why this should be a difficult problem is that the approximation bound proved in the present paper, which is of the same order as for the 
	Stochastic Block Model, is hard to use for controlling the perturbation of the eigenspaces of $Z$. More precise use of the inherent randomness of the 
	perturbation, in the spirit of \cite{vu2011singular}, might bring the necessary ingredient in order to go a little 
	further in this direction.

	\appendix
	\section{Concentration inequalities}\label{app:logsob}
	The following inequality is a particuliar case of the Log-Sobolev concentration inequality, see Theorems 5.5 and 5.6. in \cite{Boucheron2013}.
	\begin{theorem}[Gaussian concentration inequality]\label{thm:lip}\ 
		Let $Y_1,\ldots,Y_n$ be independent Gaussian random vectors on $\mathbb{R}^p$ with mean $0$ and variance $I_p$. 
		Assume that $F:\mathbb{R}^{n\times p}\to \mathbb{R}$ is Lipschitz with constant $L$, i.e.
		\[
		|F(y')-F(y)|\leq L\|y'-y\|_2\quad  \mbox{for all}\  y,y'\in \mathbb{R}^{n\times p}.
		\]
		Then the random variable $F=F(Y_1,\ldots,Y_n)$ satisfies 
		\[
		\mathrm{E}[ \exp(\theta \left( F - \mathrm{E}F \right))] \leq \exp(L^2\theta^2/2) \quad \mbox{for all } \theta\in\mathbb{R}
		\] 
		and also
		\[
		\mathrm{P}(|F-\mathrm{E}F|>t)\leq 2\exp\left(-t^2/(8L^2) \right)\quad \mbox{for all } t>0.
		\]
	\end{theorem}
	
	The next theorem provides result for the expected maxima of (non necessarily independent) subgaussian random variables.
	
	\begin{theorem}\label{thm:Dudley}\
		Let $Z_1, \cdots, Z_N$ be real valued sub-Gaussian random variables with variance factor $\nu$, i.e. satisfying
		\[
		\mathrm{E}[\exp(\theta Z_i)] \leq \exp(\nu\theta^2/2) \quad \mbox{for all } \theta\in\mathbb{R}.
		\] 
		Then 
		\[
		\mathrm{E}\left[ \max_{i=1,\cdots, N} Z_i \right]\leq \sqrt{2 \nu \log N} .
		\]
	\end{theorem}
	
	\section{The Grothendieck inequality}
	In this paper, we use the following matrix version of Grothendieck inequality.
	We denote by $\mathcal{M}_G$ the set of matrices  $Z=XY^T$ with $X,Y\in\mathbb{R}^{n\times n}$ having all raws in the unit Euclidean ball, i.e.
	\[
	\forall i\in\{1,\ldots,n\},\quad \sum_{j=1}^n X_{ij}^2\leq 1\quad \mbox{and}\quad \sum_{j=1}^n Y_{ij}^2\leq 1
	\]  
	\begin{theorem}[Grothendieck inequality]\label{theoG}
		There exists an universal constant $K_G$ such that every matrix $B\in\mathbb{R}^{n\times n}$ satisfies
		\[
		\max_{Z\in\mathcal{M}_G} |\langle B,Z\rangle|\leq K_G \|B\|_{\infty\to 1}
		\]
		where the $\ell^{\infty}\rightarrow \ell^1$ norm of $B$ is defined by \eqref{eq:inftyto1}.
	\end{theorem}
	It is also useful to note the following properties of $\mathcal{M}_G$, see Lemma 3.3 in \cite{guedon2015community}.
	\begin{lemma}\label{lemG}
		Every matrix $Z\in\mathbb{R}^{n\times n}$ such that $Z\succeq 0$ and $\mathrm{diag}(Z)\leq  1_n$ satisfies $Z\in\mathcal{M}_G$.
	\end{lemma}
	
	\bibliographystyle{amsplain}
	\bibliography{BiblioClust}{}

\providecommand{\bysame}{\leavevmode\hbox to3em{\hrulefill}\thinspace}
\providecommand{\MR}{\relax\ifhmode\unskip\space\fi MR }
\providecommand{\MRhref}[2]{%
  \href{http://www.ams.org/mathscinet-getitem?mr=#1}{#2}
}
\providecommand{\href}[2]{#2}
\begin{thebibliography}{10}

\bibitem{abbe2014exact}
Emmanuel Abbe, Afonso~S Bandeira, and Georgina Hall, \emph{Exact recovery in
  the stochastic block model}, arXiv preprint arXiv:1405.3267 (2014).

\bibitem{akaike1974new}
Hirotugu Akaike, \emph{A new look at the statistical model identification},
  Automatic Control, IEEE Transactions on \textbf{19} (1974), no.~6, 716--723.

\bibitem{bandeira2015ten}
Afonso~S Bandeira, \emph{Ten lectures and forty-two open problems in the
  mathematics of data science},  (2015).

\bibitem{belkin2001laplacian}
Mikhail Belkin and Partha Niyogi, \emph{Laplacian eigenmaps and spectral
  techniques for embedding and clustering.}, NIPS, vol.~14, 2001, pp.~585--591.

\bibitem{biernacki2000assessing}
Christophe Biernacki, Gilles Celeux, and G{\'e}rard Govaert, \emph{Assessing a
  mixture model for clustering with the integrated completed likelihood},
  Pattern Analysis and Machine Intelligence, IEEE Transactions on \textbf{22}
  (2000), no.~7, 719--725.

\bibitem{biernacki2003degeneracy}
Christophe Biernacki and St{\'e}phane Chr{\'e}tien, \emph{Degeneracy in the
  maximum likelihood estimation of univariate gaussian mixtures with em},
  Statistics \& probability letters \textbf{61} (2003), no.~4, 373--382.

\bibitem{Boucheron2013}
St{\'e}phane Boucheron, G{\'a}bor Lugosi, and Pascal Massart,
  \emph{Concentration inequalities}, Oxford University Press, Oxford, 2013, A
  nonasymptotic theory of independence, With a foreword by Michel Ledoux.
  \MR{3185193}

\bibitem{boyd2004convex}
Stephen Boyd and Lieven Vandenberghe, \emph{Convex optimization}, Cambridge
  university press, 2004.

\bibitem{candes2011robust}
Emmanuel~J Cand{\`e}s, Xiaodong Li, Yi~Ma, and John Wright, \emph{Robust
  principal component analysis?}, Journal of the ACM (JACM) \textbf{58} (2011),
  no.~3, 11.

\bibitem{cannings2015random}
Timothy~I Cannings and Richard~J Samworth, \emph{Random projection ensemble
  classification}, arXiv preprint arXiv:1504.04595 (2015).

\bibitem{Celeux92}
Gilles Celeux and G{\'e}rard Govaert, \emph{A classification {EM} algorithm for
  clustering and two stochastic versions}, Comput. Statist. Data Anal.
  \textbf{14} (1992), no.~3, 315--332. \MR{1192205 (93k:62126)}

\bibitem{chen2015convex}
Gary~K Chen, Eric~C Chi, John Michael~O Ranola, and Kenneth Lange, \emph{Convex
  clustering: An attractive alternative to hierarchical clustering}, PLoS
  Comput Biol \textbf{11} (2015), no.~5, e1004228.

\bibitem{Dempster1977}
A.~P. Dempster, N.~M. Laird, and D.~B. Rubin, \emph{Maximum likelihood from
  incomplete data via the {EM} algorithm}, J. Roy. Statist. Soc. Ser. B
  \textbf{39} (1977), no.~1, 1--38, With discussion. \MR{0501537 (58 \#18858)}

\bibitem{guedon2015community}
Olivier Gu{\'e}don and Roman Vershynin, \emph{Community detection in sparse
  networks via grothendieck's inequality}, Probability Theory and Related
  Fields (2015), 1--25.

\bibitem{heimlicher2012community}
Simon Heimlicher, Marc Lelarge, and Laurent Massouli{\'e}, \emph{Community
  detection in the labelled stochastic block model}, arXiv preprint
  arXiv:1209.2910 (2012).

\bibitem{helmberg2000spectral}
Christoph Helmberg and Franz Rendl, \emph{A spectral bundle method for
  semidefinite programming}, SIAM Journal on Optimization \textbf{10} (2000),
  no.~3, 673--696.

\bibitem{hocking2011clusterpath}
Toby~Dylan Hocking, Armand Joulin, Francis Bach, and Jean-Philippe Vert,
  \emph{Clusterpath an algorithm for clustering using convex fusion penalties},
  28th international conference on machine learning, 2011, p.~1.

\bibitem{jain2010data}
Anil~K Jain, \emph{Data clustering: 50 years beyond k-means}, Pattern
  recognition letters \textbf{31} (2010), no.~8, 651--666.

\bibitem{johnson1984extensions}
William~B Johnson and Joram Lindenstrauss, \emph{Extensions of lipschitz
  mappings into a hilbert space}, Contemporary mathematics \textbf{26} (1984),
  no.~189-206, 1.

\bibitem{jolliffe2002principal}
Ian Jolliffe, \emph{Principal component analysis}, Wiley Online Library, 2002.

\bibitem{lee2014multiway}
James~R Lee, Shayan~Oveis Gharan, and Luca Trevisan, \emph{Multiway spectral
  partitioning and higher-order cheeger inequalities}, Journal of the ACM
  (JACM) \textbf{61} (2014), no.~6, 37.

\bibitem{linial1995geometry}
Nathan Linial, Eran London, and Yuri Rabinovich, \emph{The geometry of graphs
  and some of its algorithmic applications}, Combinatorica \textbf{15} (1995),
  no.~2, 215--245.

\bibitem{mclachlan2004finite}
Geoffrey McLachlan and David Peel, \emph{Finite mixture models}, John Wiley \&
  Sons, 2004.

\bibitem{mossel2012stochastic}
Elchanan Mossel, Joe Neeman, and Allan Sly, \emph{Stochastic block models and
  reconstruction}, arXiv preprint arXiv:1202.1499 (2012).

\bibitem{overton2009hanso}
M~Overton, \emph{Hanso: a hybrid algorithm for nonsmooth optimization},
  Available from cs. nyu. edu/overton/software/hanso (2009).

\bibitem{radchenko2014consistent}
Peter Radchenko and Gourab Mukherjee, \emph{Consistent clustering using an
  $\ell\_1 $ fusion penalty}, arXiv preprint arXiv:1412.0753 (2014).

\bibitem{schwarz1978estimating}
Gideon Schwarz et~al., \emph{Estimating the dimension of a model}, The annals
  of statistics \textbf{6} (1978), no.~2, 461--464.

\bibitem{tan2015statistical}
Kean~Ming Tan, Daniela Witten, et~al., \emph{Statistical properties of convex
  clustering}, Electronic Journal of Statistics \textbf{9} (2015), no.~2,
  2324--2347.

\bibitem{Tropp:SODA09}
Joel~A Tropp, \emph{Column subset selection, matrix factorization, and
  eigenvalue optimization}, Proceedings of the Twentieth Annual ACM-SIAM
  Symposium on Discrete Algorithms, Society for Industrial and Applied
  Mathematics, 2009, pp.~978--986.

\bibitem{von2007tutorial}
Ulrike Von~Luxburg, \emph{A tutorial on spectral clustering}, Statistics and
  computing \textbf{17} (2007), no.~4, 395--416.

\bibitem{vu2011singular}
Van Vu, \emph{Singular vectors under random perturbation}, Random Structures \&
  Algorithms \textbf{39} (2011), no.~4, 526--538.

\bibitem{wang2016sparse}
Binhuan Wang, Yilong Zhang, Wei Sun, and Yixin Fang, \emph{Sparse convex
  clustering}, arXiv preprint arXiv:1601.04586 (2016).

\bibitem{weinberger2006unsupervised}
Kilian~Q Weinberger and Lawrence~K Saul, \emph{Unsupervised learning of image
  manifolds by semidefinite programming}, International Journal of Computer
  Vision \textbf{70} (2006), no.~1, 77--90.

\end{thebibliography}


\end{document}